\documentclass[letterpaper]{article}
\usepackage{aaai}
\usepackage{times}
\usepackage{helvet}

\usepackage{courier}
\usepackage{microtype}
\usepackage{graphicx}
\usepackage{booktabs} 
\usepackage{bookmark}
\usepackage{caption}
\usepackage{subfigure}

\newcommand{\citet}[1]{\citeauthor{#1}~\shortcite{#1}}
\newcommand{\citep}{\cite}

\usepackage{latexsym}
\usepackage{pifont}
\usepackage{dirtytalk}
\usepackage{amsmath}
\usepackage{amsfonts}
\usepackage{array,multirow}
\usepackage{mathtools}
\usepackage{math_commands}
\usepackage{url}

\usepackage{amsthm}

\usepackage{enumitem}
\newlist{Properties}{enumerate}{2}
\setlist[Properties]{label=Property \arabic*.,itemindent=*}

\newcommand{\mat}[1]{\text{\bf #1}}

\def\tr{\qopname\relax n{tr}}

\usepackage{adjustbox}

\begin{document}
%
\title{Compressing Deep Neural Networks via Layer Fusion}
\author{James O' Neill\\
Department of Computer Science \\
University of Liverpool\\
Liverpool, L69 3BX England\\
}

\author{James O' Neill,\textsuperscript{1}
Greg Ver Steeg\textsuperscript{2} \& 
Aram Galstyan \textsuperscript{2}\\
\textsuperscript{1}{Department of Computer Science,
University of Liverpool, Liverpool, L69 3BX England}\\
\textsuperscript{2}{USC Information Sciences Institute, Marina del Rey, California, 90292, United States of America}\\
\texttt{james.o-neill@liverpool.ac.uk} \\
\{\texttt{gregv, galstyan}\}\texttt{@isi.edu}}

\maketitle
\begin{abstract}
\begin{quote}
This paper proposes \textit{layer fusion} - a model compression technique that discovers which weights to combine and then fuses weights of similar fully-connected, convolutional and attention layers. Layer fusion can significantly reduce the number of layers of the original network with little additional computation overhead, while maintaining competitive performance. From experiments on CIFAR-10, we find that various deep convolution neural networks can remain within 2\% accuracy points of the original networks up to a compression ratio of 3.33 when iteratively retrained with layer fusion. For experiments on the WikiText-2 language modelling dataset where pretrained transformer models are used, we achieve compression that leads to a network that is 20\% of its original size while being within 5 perplexity points of the original network. We also find that other well-established compression techniques can achieve competitive performance when compared to their original networks given a sufficient number of retraining steps. Generally, we observe a clear inflection point in performance as the amount of compression increases, suggesting a bound on the amount of compression that can be achieved before an exponential degradation in performance. 
\end{quote}
\end{abstract}

\section{Introduction}\label{sec:intro}
Deep neural networks (DNNs) have made a significant impact in fields such as Computer Vision (CV)~\citep{he2016deep,densenet} and Natural Language Processing (NLP) ~\citep{vaswani2017attention,devlin2018bert}. This has been accelerated due to numerous innovations. For example, Residual Networks (ResNets) ~\citep{he2016deep}, that are often employed in CV, use skip connections to avoid the vanishing gradient problem in very deep networks, while batch normalization  ~\citep{ioffe2015batch,ba2016layer} and layer normalization~\citep{ba2016layer} are used to reduce the effects of shifts in the training and test data distributions. Tangentially for NLP, Transformer networks have shown great success due to the use of self-attention~\citep{vaswani2017attention}. 
Transformers have shown significant performance improvements over Recurrent Neural Networks with internal memory (RNNs)~\citep{hochreiter1997long} for sentence representations, language modelling and conditional text generation~\citep{radford2018improving,dehghani2018universal,dai2019transformer} and various other natural language understanding (NLU) tasks. Similarly, deep Convolutional Neural Networks (CNNs)~\citep{krizhevsky2012imagenet} have improved performance on image classification~\citep{krizhevsky2012imagenet}, image segmentation~\citep{long2015fully}, speech recognition~\citep{lecun1995convolutional} and been widely adopted in the machine learning (ML) community.
However, large overparameterized networks require more compute, training time, storage and leave a larger carbon footprint~\citep{strubell2019energy}. Previous work on model compression has mainly focused on deploying compressed models to mobile devices~\citep{han2015deep,wu2016quantized}. However, moving models from multi-GPU training to single-GPU training is now too a salient challenge, in order to relax the resource requirements for ML practitioners and allow a wider adoption of larger pretrained CNNs and Transformers within the community.
DNNs becoming increasingly deeper leads us to ask the following questions: \textit{are all layers of large pretrained models necessary for a given target task ? If not, can we reduce the network while preserving network density during retraining in a computationally efficient manner?}.
We are also motivated to fuse layers based on findings that show whole layers can be distinctly separated by their importance in prediction~\citep{zhang2019all}. Earlier work on CNNs found that some layers may become redundant in very deep networks~\cite{he2016deep}, essentially copying earlier layers and performing identity mappings for the redundant layers. While residual connections have ameliorated these problems to some degree (not only in residual networks e.g Transformers), we assert that there may still be significant overlap between layers of large overparameterized networks.


Guided by Occam's razor~\citep{blumer1987occam}, one can use various compression techniques (e.g pruning, tensor decomposition, knowledge distillation, quantization) post-training to find smaller networks from pretrained models. However, many of the existing compression techniques are unstructured~\citep{karnin1990simple,hassibi1993second,han2015learning}, resulting in a sparse model. This is a practical limitation since sparse networks require more conditional operations to represent which elements within each parameter matrix are zero or non-zero. Current GPU libraries such as CuSPARSE (accounting for recent improvements~\citep{argueta2019accelerating}) are far slower than CuBLAS~\citep{sanders2010cuda} and fundamentally, current hardware is not designed to optimize for sparse matrix operations.
In contrast, knowledge distillation~\citep{hinton2015distilling,mishra2017apprentice,ashok2017n2n}, and quantization~\citep{polino2018model} preserve network density, avoiding the necessity for specialized sparse matrix libraries to utilize the benefits of smaller and faster networks. However, quantization leads to quantization error and requires approximate methods to compute partial derivatives during retraining~\citep{agustsson2017soft} and knowledge distillation requires more memory to store and train the smaller network.
Weight sharing reduces the network size and avoids sparsity, however it is unclear which weights should be shared and it cannot be used when the model is already pretrained with untied weights.
The noted drawbacks of the aforementioned compression methods further motivates to seek an alternative structured compression method that preserves network density while identifying and removing redundancy in the layers. This brings us to our main contributions.

\paragraph{Contributions}
We propose \textit{layer fusion} (LF). LF aims to preserve information across layers during retraining of already learned models while preserving layer density for computational efficiency. 
Since aligning paths, layers and whole neural networks is non-trivial (neurons can be permuted and still exhibit similar behaviour and we desire an invariance to orthogonal transformations), we also propose \textit{alignment}  measures for LF. This includes (1) a Wasserstein distance metric to approximate the alignment cost between weight matrices and (2) numerous criteria for measuring similarity between weight covariance matrices. We use these measures as LF criteria to rank layer pairs that are subsequently used to fuse convolutional, fully-connected and attention layers. This leads to both computational and performance improvements over layer removal techniques, network pruning and shows competitive results compared to tensor-decomposition and unsupervised-based knowledge distillation. In our experiments, we report LF using different fusion approaches: layer freezing, averaging and random \textit{mixing}. 
 %
%
Lastly, we report results on using structured compression for large pretrained transformers and CNNs and provide experimental results of different compression methods with and without retraining. Thus, we identify the importance of retraining pretrained models.  

\section{Related Work}\label{sec:struct_rep_sim}

\subsection{Layer Structure \& Importance}
~\citet{zhang2019all} have recently analysed the layer-wise functional structure of overparameterized deep models to gain insight into why deep networks have performance advantages over their shallow counterparts. They find that some layers are salient and that once removed, or reinitialized, have a catastrophic effect on learning during training and subsequently generalization. In contrast, the remaining layers once reset to their default initialization has little effect. This suggests that parameter and norm counting is too broad of a measure to succinctly study the generalization properties in deep networks. 
These findings also motivate LF, as we posit that important layers are more distinct and therefore will be less similar, or harder to align with other layers, while more redundant layers may be good candidates for fusing layers. 

Recently,~\citet{frankle2018lottery} empirically showed that there exists trained subnetworks that when re-initialized to their original configuration produce the same performance as the original network in the same number of training epochs. They also posit that stochastic gradient descent (SGD) seeks out a set of \textit{lottery tickets} (i.e well-initialized weights that make up a subnetwork that when trained for the same number of epochs as the original network, or less, can reach the same out-of-sample performance) and essentially ignores the remaining weights. 
We can further conjecture from~\citet{zhang2019all} findings, that perhaps SGD more generally seeks out important layers, which we analogously refer to as \textit{lottery pools}. Identifying whole layers that are significantly distinguished from others, in terms of their influence on learning, further motivates us to pursue the merging or freezing of layers in DNNs. 

\subsection{Computing Layer Similarity}
~\citet{kornblith2019similarity} have focused on computing similarity between different neural representations (i.e the activation output vector for a given layer). However, we view this comparison between layers as slightly limiting, since information is lost about what weights and bias led to the activation outputs. Moreover, directly comparing neural network weights allows us to avoid sampling inputs to compute the activations. In contrast, work focusing on representational similarity across networks \citet{li2016convergent,kornblith2019similarity}, we are instead comparing weight matrices within the same network. Directly comparing weights and biases allow us to better approximate alignments and similarities for dense networks and has the advantage that we do not require data to be fed-forward through the network post-training to measure similarity within or across networks, unlike representational similarity (i.e output activations).

Structured Dropout~\citet{fan2019reducing} proposes to randomly drop whole layers during training time and at test time they can choose a subnetwork which can be decided based on performance of different combinations of pruned networks on the validation set or based on dropout probabilities learned for each layer throughout training. 

~\citet{singh2019model} measure model similarity across neural networks using optimal transport-based metrics. In contrast, our work measures intra-network similarity and we make a distributional assumption that allows us to use such metric to be used efficiently during retraining and to scale for large networks. 


\section{Methodology}

\textbf{Preliminaries}
We define a dataset as $\{D = (\vec{x}_i, \vec{y}_i): i = 1, \ldots T\}$ that contains $T$ tuples of an input vector $\vec{x} \in \R^n$ and a corresponding target $\vec{y} \in \{0 ,1\}^p$. We define any arbitary sample as $s := (\vec{x}, \vec{y})$ where $s \in D$. We consider a neural network $f_{\theta}(\vec{x})$ with pretrained parameters $\theta := ( \mat{W}_1,  \mat{W}_2,\ldots \theta_{\ell} \ldots, \theta_L)^T$. Here $\theta_{\ell}:= \{\mat{W}_{\ell}, \vec{b}_{\ell}\}$ where $\mat{W}_{\ell} \in \mathbb{R}^{n_{\ell} \times n_{\ell+1}}$, $\vec{b} \in \mathbb{R}^{n_{\ell+1}}$ where $n_\ell$ denotes the dimension size of the $\ell$-th layer. Thus, a standard fully-connected $f_{\theta}$ is expressed as,

\begin{equation}
f_{\theta}(\vec{x}) := \mat{W}_{L}g\Big( \ldots g\big(\mat{W}_{2}g (\mat{W}_{1}\vec{x} + \vec{b}_1\big) + \vec{b}_2 \Big) + \vec{b}_{L}
\end{equation}

with smooth asymptotic nonlinear function $g(\cdot)$ (e.g hyperbolic tangent) that performs elementwise operations on its input. The input to each subsequent layer as $\vec{z}_{\ell} \in \mathbb{R}^{n_{\ell}}$ where $\vec{x} \coloneqq \vec{z}_0$ for $m$ number of units in layer ${\ell}$ and the corresponding output activation as $\vec{T}_{\ell} = g(\vec{z}_{\ell})$. 
The loss function is defined as $\mathbb{L}_{\theta}(D):= \frac{1}{T}\sum_{i=1}^N \mathcal{L}(\vec{y}_i, f_\theta(\vec{x}_i))$ where for a single sample $s_i$, $\mathcal{L}: \mathcal{Y} \times \mathbb{R}^n \to \mathbb{R}$. 
A pruned $\theta_{\ell}$ post-training is denoted as $\theta^p_{\ell}$ and a tensor decomposed $\theta_{\ell}$ is expressed as $\tilde{\theta}_{\ell}$ where $\tilde{\mat{W}}_{\ell} \in \R^{d_{\ell} \times d_{\ell+1}}$ and $\tilde{\vec{b}}_{\ell} \in \R^{d_{\ell+1}} $ and $d \ll n$. 
A network pruned by layer as a percentage of the lowest weight magnitudes is denoted as $f^{lp}_{\acute{\theta}}$ where the pruned weights $\acute{\theta} \subset \theta$. A percentage of the network pruned by weight magnitudes across the whole network is denoted as $f^{gp}_{\acute{\theta}}$ (i.e global pruning). Lastly, a post layer fused network $f_{\Theta}$ has fused parameters $\Theta$.

\subsection{Desirable Properties of a Layer Similarity Measure}
Ideally, we seek a measure that can compare weight matrices that are permutable and of varying length. Formally, the main challenges with aligning weight matrices $\mat{W}: = \{\mat{W}_0,\ldots, \mat{W}_{\ell}, \ldots \mat{W}_{L}\}$ of different layers is that, when vectorized as $\mathtt{vec}(\mat{W}_{\ell}) \in \mathbb{R}^{n_{\ell}(n_{\ell+1})}$, $\mat{W}_{\ell}$ can be permuted and still exhibit the same behavior at the output. Hence, if $|\mat{W}_{\ell}| \neq |\mat{W}_{\ell+1}|$, the measure $S$ must allow for multisets of different lengths and permutations. 
Invariance to rotations, reflections and scaling are all desirable properties we aim to incorporate into measuring similarity between weight matrices. However, invariance to linear transformations has issues when there are more parameters in a layer than training samples, as pointed out by~\citet{kornblith2019similarity}.
Eventhough our work mainly focuses on large pretrained models, we too seek a LF measure that is invariant to orthogonal transformations to overcome the aforementioned issues i.e for a similarity function $s(\cdot, \cdot)$, $s(\mat{W}_{i}, \mat{W}_{j}) =
s(\mat{W}_{i} \mat{U}, \mat{W}_{j} \mat{V})$ for full-rank orthonormal matrices $\mat{U}$ and $\mat{V}$ such that $\mat{U}^{T}\mat{U} = \mat{I}$ and $\mat{V}^{T}\mat{V} = \mat{I}$. More importantly, invariance to orthogonal transformation relates to permutation invariance~\citep{orhan2017skip} which is a property we account when measuring similarity to fuse weight matrices. We now describe a set of measures we consider for aligning and measuring the similarity of layers. 


\subsection{Layer Alignment \& Layer Similarity}\label{sec:lalignment}

\paragraph{Covariance Alignment}
The first layer fusion measure we consider is covariance alignment (CA). CA accounts for correlated intra-variant distances between layers, which can indicate some redundancy, although their overall distributions may differ and therefore may be good candidates for LF. Hence, we consider the Frobenius norm (denoted as subscript $F$) between pairs of weight covariance matrices $\vec{\Sigma}_{\tilde{\mat{W}}_1}, \vec{\Sigma}_{\tilde{\mat{W}}_2}$ and expectation $\mathbb{E}[{\mat{W}}_1] = \mathbb{E}[{\mat{W}}_2] =0$. This forms a Riemannian manifold of non-positive curvature over the weight covariances. We first consider the cosine distance as the distance measures between parameter covariance matrices as Equation \ref{eq:cos_cov}, where $||\Sigma_{\mat{W}}||_F = [\text{tr}(\Sigma_{\mat{W}}^{T}\Sigma_{\mat{W}})]^{1/2}$.

\begin{equation}\label{eq:cos_cov}
\text{cos}(\vec{\Sigma}_{\mat{W}_1}, \vec{\Sigma}_{\mat{W}_2}) = \frac{\tr\big(\vec{\Sigma}_{\mat{W}_1} \cdot \vec{\Sigma}_{\mat{W}_2}\big)}{||\vec{\Sigma}_{\mat{W}_1}||_F ||\vec{\Sigma}_{\mat{W}_2}||_F}
\end{equation}

If we assume both weight matrices are drawn from a normal distribution $\mat{W}_{1} \sim \mathcal{N}(\mu, \sigma_1), \mat{W}_{2} \sim \mathcal{N}(\mu, \sigma_2)$ with identical means $\mu= \mu_{\mat{W}_1}= \mu_{\mat{W}_2}$, the KL divergence between their covariance matrices can be expressed as:

\begin{equation}\label{eq:kl_cov}
\text{KL}(\mathbf{\Sigma}_{\mat{W}_1}||  \mathbf{\Sigma}_{\mat{W}_2}) = \frac{1}{2}\Big[\tr \big(\vec{\Sigma}_{\mat{W}_2}^{-1} \vec{\Sigma}_{\mat{W}_1})-\ln \big(\frac{|\vec{\Sigma}_{\mat{W}_1}|}{|\vec{\Sigma}_{\mat{W}_2}|}\big)\Big] 
\end{equation}

The symmetrized KL divergence between positive semi-definite matrices (e.g covariances) also acts as the square of a distance~\cite{moakher2006symmetric} (see supplementary for further details, including descriptions of other covariance similarity measures).
We consider both Equation \ref{eq:cos_cov} and Equation \ref{eq:kl_cov} for fusing convolutional layers, self-attention layers and fully-connected layers. The KL is an asymmetric measure, therefore the divergence in both directions can be used to assign a weight to each layer pair in layer fusion. 

\paragraph{Optimal Transport \& Wasserstein Distance}
Unlike an all-pair distance such as CA, Wasserstein (WS) distance can also be used to find the optimal cost, also known as the optimal flow between two distributions. Unlike, other distance measures, WD tries to keep the geometry of the distributions intact when interpolating and measuring between distributions. Unlike CA and other baseline measures, WS is invariant to layer permutations and like CA, it also accounts for mutual dependencies between parameters in any arbitrary layer.
In this work, we consider the WD between adjacent row-normalized parameter pairs $\text{softmax}(\mat{W}_{1}, \mat{W}_{2})$ (i.e multisets) in a Euclidean metric space. Given two multi-sets $\mat{W}_{1}, \mat{W}_{2} \subset \mat{W}$, of size $d=|\mat{W}_{1}|=|\mat{W}_{2}|$ with corresponding empirical distributions $P_{\mat{W}_1}$ and $P_{\mat{W}_2}$, the WS distance is defined as Equation \ref{eq:wass}. However, computing WS distance is $\mathcal{O}(N^3)$ using the standard Hungarian algorithm, which is intractable for large $\theta$.


\begin{equation}\label{eq:wass}
W_p(P_{\mat{W}_1}, P_{\mat{W}_2}) = \inf_{\pi}\Big(\sum_{i=1}^d ||P_{\mat{W}_{1}^{i}} - P_{\mat{W}_{2}^{\pi(i)}}||^{p} \Big)^{1/p}
\end{equation}

One way to tradeoff this computational burden is to assume that the weights are i.i.d and normally distributed at the expense of disregarding mutual dependencies learned throughout training. 
According to Lyapunov’s central limit theorem~\citep{lehmann2004elements}, we can assume the the weights in a layer are normally distributed. 
Hence, if $P_{\mat{W}_1} = N( \mu_{\mat{W}_1}, \vec{\Sigma}_1)$ and $P_{\mat{W}_2} = N( \mu_{\mat{W}_2}, \vec{\Sigma}_2)$ we can express the 2-WS distance as Equation \ref{eq:wasser}, also known as the Bures metric\footnote{Often used in quantum physics for measuring quantum state correlations~\citep{forrester2016relating}.}.

\begin{equation}\label{eq:wasser}
\begin{aligned}
\mathbb{W}^2(P_{ \mat{W}_1}, P_{ \mat{W}_2}) = ||\mu_{ \mat{W}_1} - \mu_{ \mat{W}_2}||^2 + \mathbb{B}^2(\vec{\Sigma}_{ \mat{W}_1}, \vec{\Sigma}_{ \mat{W}_2}), \\
\mathbb{B}^2(\vec{\Sigma}_{ \mat{W}_1}, \vec{\Sigma}_{ \mat{W}_2}) = \text{tr}(\vec{\Sigma}_{ \mat{W}_1}) + \text{tr}(\vec{\Sigma}_{ \mat{W}_2}) - \\
2\text{tr}\Big[\big(\sqrt{\vec{\Sigma}_{ \mat{W}_1}} (\vec{\Sigma}_{ \mat{W}_2}\sqrt{\vec{\Sigma}_{ \mat{W}_1}})\big)\Big]
\end{aligned}
\end{equation}

Although we focus on the Bures metric in our experiments, an alternative approach is to find a set of cluster centroids in $ \mat{W}_1$ and $ \mat{W}_2$ as $C_{ \mat{W}_2}$ and $C_{ \mat{W}_2}$ and compute $W(P_{C_{ \mat{W}_1}}, P_{C_{ \mat{W}_2}})$. In this approach the centroids are converted to an empirical distribution $P_{C_{\theta}}$ such $c \ll d$ such that a $\mathcal{O}(N^3)$ cost is feasible for computing during retraining steps. Alternatively, we could avoid softmax normalization and directly compute $W(C_{ \mat{W}_1}, C_{ \mat{W}_2})$ on both discrete sets. Lastly, we note that when fusing layers with 2-WS distance, the fusion occurs between aligned weights given by the cost matrix. Hence, it is not only used to identify top-k most similar layers, but the cost matrix also aligns which weights in the layer pair are fused.

\subsection{Fusing Layers}\label{sec:lfusion}

After choosing the top-k layer pairs to merge, we then consider 3 ways to fuse the layers: (1) freeze one of the two layers (i.e do not compute gradients for one of the two layers), (2) take the mean between layer pairs and compute backprop on the averaged layer pair and (3) sample and mix between the layers. Choosing the layers to fuse for (1) is based on which of the two is closest to the middle layer of the network. This is motivated by previous work that showed layers closer to the input and output are generally more salient~\citep{zhang2019all}. When using Jenson-Shannon divergence for choosing top-k layers, we use the divergence asymmetry for choosing which layer is frozen. This is achieved by taking the parameter $\gamma$ between the Jenson-Shannon divergence of two layers in both directions to control a weighted gradient. We express the backpropogation when using LF with Jenson-Shannon divergence in terms of KL-divergences as shown in Equation \ref{eq:layer_fused_backprop}, where $\bar{\mat{W}}_{ij}$ is a mixture distribution between ${\mat{W}}_i$ and ${\mat{W}}_j$ with a weighted gradient $\partial \mathcal{L}/\partial\tilde{\mat{W}}_{ij}$ that represents the gradient for both ${\mat{W}}_i$ and ${\mat{W}}_j$. Thus, for the backward pass of a frozen layer from a given top-k pair, we still compute its gradients which will influence how its original pair will be updated. This constraint ensures that the original pair that were most similar for a given compression step remain relatively close throughout retraining. The layer pair are then averaged at test time to reduce network size, while maintaining similarity using the aforementioned JS divergence gradient constraint. 

\begin{equation}\label{eq:layer_fused_backprop}
\begin{gathered}
\frac{\partial \mathcal{L}}{\partial\tilde{\mat{W}}_{ij}} := \gamma \Big(\frac{\partial L}{\mat{W}_i} \Big) + \big(1 - \gamma \big)\Big( \frac{\partial L}{\mat{W}_j} \Big) \quad s.t, \\
\gamma = \frac{1}{2}\Big( \text{KL}(\mat{W}_i||\bar{\mat{W}}_{ij}) + \text{KL}(\mat{W}_j||\bar{\mat{W}}_{ij}) \Big)
\end{gathered}
\end{equation}

For (2), updates during training when using , we constrain the gradients to be the average of both layers and then average the resulting layers at the end of retraining. 
For (3), we interpolate between hidden representations that are most similar, which can be viewed as a stochastic approach of JS divergence used in (1), to remove redundancy in the network. We denote a pair of randomly mixed layers as $\tilde{\mat{W}}^{i}_{\ell} \sim \mathbb{B}\big( \mat{W}^{i}_{\ell}, \mat{W}_{\ell+1} \big) \quad \forall i \in n_{\ell}$. Note that we only mix between pairs of weight matrices, the bias terms are averaged $(\vec{b}_{\ell}+\vec{b}_{\ell})/2$. We then compute backpropogation on $\tilde{\theta}^{i}_{\ell}$ instead of the original unmixed layer pair $(\theta {i}_{\ell},\tilde{\theta}^{i}_{\ell+1})$ i.e mixing is carried out before the forward pass.


\section{Experimental Details}
We focus on transformer-based models for language modelling on the WikiText-2 dataset~\citep{merity2016pointer}. For large models in NLP such as BERT~\citep{devlin2018bert}, OpenAI-GPT, GPT2~\citep{radford2018improving} and Transformer-XL~\citep{yang2019xlnet}, we freeze or combine layer weights of each multi-attention head component and intermediate dense layers, dramatically reducing the respective number of layer and weights.
For image classification on CIFAR10, we report results for ResNet, ResNet-ELU~\citep{shah2016deep}, Wide-ResNet~\citep{Zagoruyko2016WRN} and DenseNet~\citep{huang2017densely}. We are particularly interested in ResNet architectures, or more generally, ones that also use skip connections. This is motivated by ~\citet{veit2016} which found that deleting or permuting residual blocks can be carried out without much degradation in performance in a pretrained ResNet.

\subsection{Compression Without Reraining}

\begin{figure*}[htp]
 \centering
 \subfigure[No Retraining]{\includegraphics[height=0.25\textheight,width=.46\textwidth]{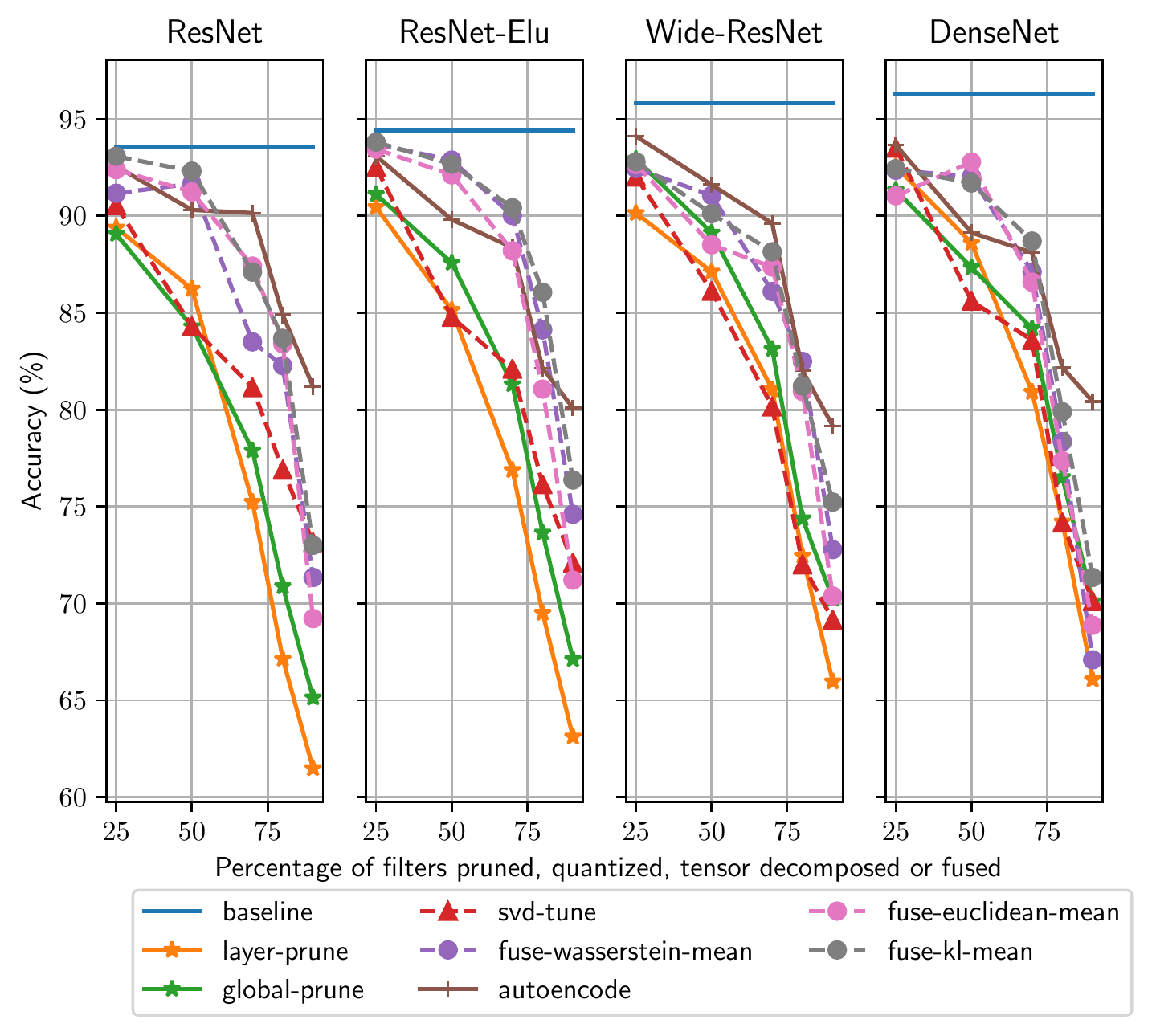}\label{fig:no_retraining_cifar}}
 \subfigure[Retraining]{\includegraphics[height=0.25\textheight,width=.46\textwidth]{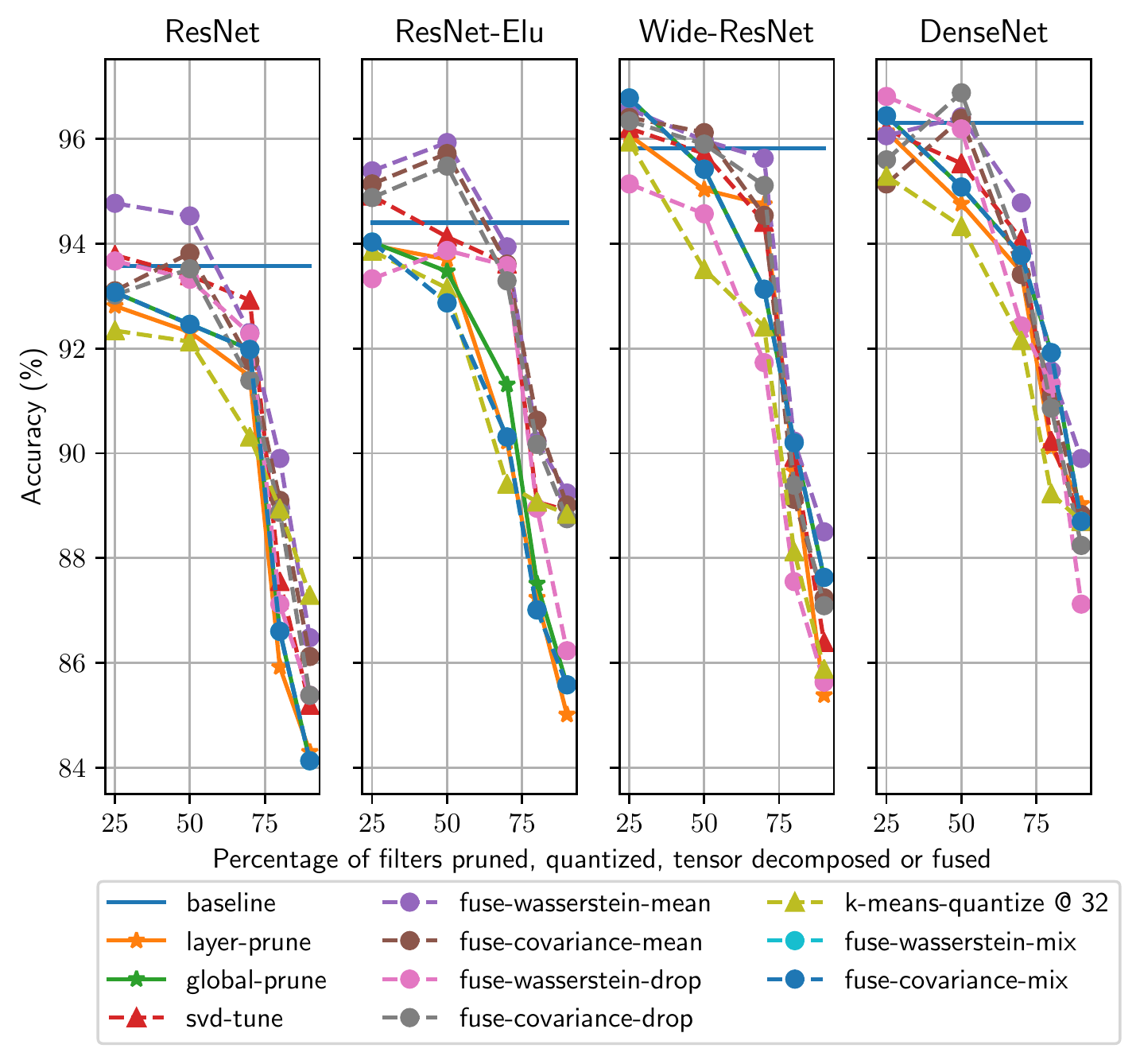}\label{fig:retraining_cifar}}
 \caption{CIFAR-10 Test Accuracy with and without Retraining}
\end{figure*}

For magnitude-based pruning, we prune a percentage of the weights with the lowest magnitude. This is done in one of two ways: a percentage of weights pruned by layer (layer pruning), or a percentage of the network as whole (global pruning). For quantization, we use k-means whereby the number of clusters for a given layer is specified as a percentage of the original size of that layer (i.e number of parameters in the tensor).
For tensor decomposition, we reduce the number of parameters by approximating layers with a lower rank using singular value decomposition (SVD). Specifically, we use randomized truncated SVD~\citep{halko2011finding} where QR factorization on $\mat{W}_{\ell}$ such that $\mat{Q}^T_{\ell} \mat{W}_{\ell} = \mat{R}_{\ell}$ where $\mat{Q}_{\ell}$ are the orthonormal columns of $\mat{W}_{\ell}$. Randomized methods are used to approximate the range of $\theta_{\ell}$ and reduce computation from $\mathcal{O}(\min(n_{\ell-1}n_{\ell}^2,n_{\ell-1}^2 n_{\ell}))$ to $\mathcal{O}(n_{\ell-1}n_{\ell} \log(k))$ where $k$ represents the approximate rank of $\theta_{\ell}$. We also perform dimensionality reduction on the layers by using 1-hidden layer denoising autoencoders which use the same activation functions for reconstruction as the original architecture and a mean squared error loss is minimized. The encoder layer of each denoising AE (DAE) is the used in replacement of the original layer. For both truncated SVD and DAE, this is carried out sequentially from bottom to top layer so that the reconstruction of a given layer $l$ also accounts for cascade approximation errors of dimensionality reduction from previous layers. We refer to this type of layer reconstruction technique as \textit{student rollout} because the pretrained teacher network is iteratively rolled out and reconstructed from the first layer to the last. 

\begin{table}
\caption{CIFAR-10 Test Accuracy with WS-Based CNN LF}
 \label{tab:cifar_wass_merge_types}
\begin{center}
\begin{small}
\begin{sc}
\begin{adjustbox}{max width=.45\textwidth}
\begin{tabular}{lc|cc|cc|cc|cc}
\toprule[2.pt]

& & \multicolumn{2}{c|}{Res} & \multicolumn{2}{c|}{Res-ELU} & \multicolumn{2}{c|}{Wide-Res} & \multicolumn{2}{c}{DenseNet} \\


\midrule
\midrule
\multicolumn{2}{l}{\textbf{Orig.}} & 93.75 & - & 94.40 & - & 95.82 & - & 96.31 & - \\

\midrule

\parbox[t]{2mm}{\multirow{4}{*}{\rotatebox[origin=c]{90}{\textbf{ \small{Mean}}}}} 
& \textbf{$25 \%$} & 92.39 & 94.77 & 93.45 & 95.39 & 92.66 & 96.57 & 91.04 & 96.06 \\
& \textbf{$50 \%$} & 91.24 & 94.53 & 92.12 & 95.93 & 88.51 & 95.97 & 92.78 & 96.42\\
& \textbf{$75 \%$} & 87.41 & 92.30 & 88.20 & 93.94 & 87.36 & 95.63 & 86.58 & 94.78\\
& \textbf{$80 \%$} & 83.40 & 89.90 & 81.06 & 90.23 & 80.94 & 90.23 & 77.36 & 88.50 \\
& \textbf{$90 \%$} & \textbf{69.22} & \textbf{86.48} & \textbf{71.20} & \textbf{89.24} & \textbf{70.38} & \textbf{89.90} & \textbf{68.87} & \textbf{91.57}\\

\midrule

\parbox[t]{5mm}{\multirow{4}{*}{\rotatebox[origin=c]{90}{\textbf{ \small{Freeze}}}}} 
& \textbf{$25 \%$} & 91.17 & 93.67 & 93.71 & 93.33 & 92.45 & 95.14 & 92.36 & 96.81\\
& \textbf{$50 \%$} & 91.67 & 93.32 & 92.88 & 93.87 & 91.06 & 94.57 & 92.03 & 96.19\\
& \textbf{$75 \%$} & 83.50 & 92.28 & 90.02 & 93.58 & 86.10 & 91.73 & 87.11 & 92.43\\
& \textbf{$80 \%$} & 82.27 & 87.12 & 84.12 & 88.95 & 82.49 & 87.55 & 78.35 & 85.63\\
& \textbf{$90 \%$} & \textbf{71.34} & \textbf{85.38} & \textbf{74.60} & \textbf{86.23} & \textbf{72.78} & \textbf{85.63} & \textbf{67.09} & \textbf{87.12} \\
\midrule

\parbox[t]{5mm}{\multirow{4}{*}{\rotatebox[origin=c]{90}{\textbf{ \small{Mixing}}}}} 
& \textbf{$25 \%$} & 93.67 & 93.22 & 93.33 & 94.03 & 95.14 & 96.78 & 96.81 & 96.44 \\
& \textbf{$50 \%$} & 93.32 & 92.46 & 93.24 & 93.87 & 94.57 & 95.42 & 96.19 & 95.08 \\
& \textbf{$70 \%$} & 92.28 & 91.98 & 90.31 & 93.58 & 91.73 & 93.13 & 92.43 & 93.78 \\
& \textbf{$80 \%$} & 84.12 & 90.60 & 85.58 & 88.95 & 87.55 & 90.20 & 91.32 & 91.92 \\
& \textbf{$90 \%$} & \textbf{73.2} & \textbf{86.13} & \textbf{73.50} & \textbf{85.58} & \textbf{80.63} & \textbf{87.63} & \textbf{78.12} & \textbf{88.70} \\

\bottomrule[2pt]

\end{tabular}
\end{adjustbox}
\end{sc}
\end{small}
\end{center}
\vskip -0.1in
\end{table}

\subsection{Layer Fusion \& Compression ReTraining}

\begin{table*}[ht]
\centering
\captionsetup{justification=centering, margin=0cm}
 \caption{WikiText-2 Test Perplexity without fine-tuning or retraining.}\label{tab:untrained_pruning_all_results}
\resizebox{1.\linewidth}{!}{%
\begin{tabular}{lc|ccc|ccc|ccc|ccc}
\toprule[2.pt]

& & Trans-XL & GPT-2 & GPT & Trans-XL & GPT-2 & GPT & Trans-XL & GPT-2 & GPT & Trans-XL & GPT-2 & GPT \\


\midrule
\midrule
\multicolumn{2}{l}{\textbf{Original}} & 21.28 & 26.61 & 67.23 & 21.28 & 26.61 & 67.23 & 21.28 & 26.61 & 67.23 & 21.28 & 26.61 & 67.23 \\

\midrule
& & \multicolumn{3}{c}{Layer Pruning via Weight Magnitude} & \multicolumn{3}{c}{Global Pruning via Weight Magnitude} & \multicolumn{3}{c}{Randomized SVD} & \multicolumn{3}{c}{Denoising AutoEncoder} \\
\midrule

& \textbf{@ $10 \%$} & 21.25 & 25.44 & 69.33 & 21.15 & 25.04 & 69.54 & 20.29 & 25.44 & 69.33 & 19.69 & 23.14 & 65.14 \\

& \textbf{@ $20 \%$} & 21.26 & 27.02 & 88.19 & 21.08 & 27.03 & 79.33 & 20.69 & 27.02 & 88.19 & 19.43 & 24.46 & 81.08\\

& \textbf{@ $30 \%$} & 22.05 & 35.87 & 1452.96 & 21.54 & 46.15 & 140.22 & 21.68 & 35.87 & 1452.96 & 20.57 & 29.07 & 921.06\\

& \textbf{@ $50 \%$} & 57.12 & 1627.22 & 3260.52 & 53.90 & 3271.52 & 2159.42 & 64.12 & 1627.22 & 3145.41 & 55.07 & 1258.05 & 2654.88\\
& \textbf{@ $70 \%$} & 3147.31 & 24946.66 & 21605.02 & 901.534 & 13464.17 & 18068.86 & 3679.13 & 26149.57 & 22140.12 & 2958.41 & 19206.78 & 15.60\\

\midrule[1.2pt]

& & \multicolumn{3}{c}{Layer Averaging (Euclidean Distance)} & \multicolumn{3}{c}{Layer Freezing (Euclidean Distance)} & \multicolumn{3}{c}{Global WS-LF} & \multicolumn{3}{c}{Adjacent WS-LF} \\
\midrule

& \textbf{@ $10 \%$} & 21.74 & 25.78 & 81.14 & 23.09 & 28.70 & 83.44 & 22.15 & 25.79 & 69.29 & 22.52 & 25.58 & 69.90 \\
& \textbf{@ $20 \%$} & 22.21 & 29.74 & 94.80 & 25.19 & 30.88 & 94.32 & 22.37 & 27.38 & 90.70 & 22.61 & 27.35 & 89.77 \\
& \textbf{@ $30 \%$} & 25.27 & 38.90 & 1903.14 & 27.81 & 40.01 & 97.11 & 24.79 & 38.18 & 1533.24 & 22.82 & 36.11 & 1493.37 \\
& \textbf{@ $50 \%$} & 62.04 & 1807.31 & 3724.47 & 64.38 & 1944.51 & 3790.12 & 61.68 & 1690.31 & 3123.39 & 59.70 & 1691.23 & 3357.02 \\
& \textbf{@ $70 \%$} & 3695.01 & 2631.52 & 29117.82 & 3583.16 & 23583.10 & 30258.78 & 3201.97 & 25130.30 & 22448.15 & 3198.16 & 25270.21 & 21732.58\\

\bottomrule[2pt]
\end{tabular}%
}
\end{table*}

%

For retraining we consider two main schemes: (1) for each retraining step we carry out network compression (e.g via pruning), retrain the resulting network and iteratively repeat until the final epoch, and (2) in the case where network compression leads to non-zero weights (e.g LF), we freeze the network weights apart from those which have been identified for LF in which case we retrain before tying. 

Layer averaging, mixing and freezing are experimented with for fusing layers. To maintain uniformity across each compression step, we prune, quantize, fuse and decompose a percentage of the weights as opposed to using other criteria such as thresholding. This ensures a consistent comparison across the compression methods (e.g thresholding weights in pruning does not have a direct equivalent to quantization or weight decomposition, unless we dynamically reduce the network size proportional to the number of weights pruned via thresholding).

\section{Results}

\subsection{Image Classification}

\textbf{No Retraining} Figure \ref{fig:no_retraining_cifar} shows the results of pruning, quantization, weight decomposition and our proposed LF without any retraining. A general observation is that an exponential decline in performance occurs at around 70\% (some variance depending on the compression method) of the original network is compressed. For example, fusing layers using the WS distance for alignment allows accuracy to be closer to the original networks accuracy up to 70\%. In contrast, pruning convolutional layers in ResNet models leads to a faster accuracy drop. This is somewhat surprising given that unstructured pruning is less restrictive, when applying LF to CNN architectures. We also allow filters from the same layer to be fused, in comparison to dense layers in self-attention blocks for Transformers.

\textbf{Retraining} In Figure \ref{fig:retraining_cifar} we see the results of model compression methods retraining on CIFAR-10 for ResNet-50, ResNet-50 with exponential linear units (ELUs), Wide-ResNet and DenseNet. We test each combination of layer pairs for averaging layers as $\tilde{\theta}_{i} = \tilde{\theta}_{j} = (\theta_{i} + \theta_{j}/2)$ where $\binom{L}{2}$ are the total number of layers (e.g 24 layers results in 276 possible pairs). The performance change is measured from the original network when layer averaging by choosing the top $L \times \%$ and measuring which averaged layer pair produced the smallest difference in accuracy when compared to the original network. In the case that the same layer within the top $L \times \%$ is coupled with more than one other layer, we simply take the mean of multiple pairs. This reduces computation to $2\binom{L}{L \times \% }$.

We find a re-occuring pattern that early on, retraining up to a reduction of 25\% of the network improves the results, and even up to 25\% - 50\% in some cases (e.g global pruning and layer pruning). From 75\% we see a significant decrease in performance, typically 2-4\% drop in accuracy percentage points across each model. Given an allowance of $N$ retraining epochs, allocating the amount of model compression for each compression step is a critical hyperparameter. Concretely, less retraining time is necessary for during initial model compression, whereas past a compression ratio of 3.33 (i.e 75\%), the interval between retraining steps should become larger. This is highlighted in bold across Table \ref{tab:cifar_wass_merge_types}, where left subcolumns are with no retraining and right subcolumns are with retraining. For all fusion types (mean, freezing and mixing), we find a significant increase in accuracy after retraining. Mean layer fusion using the WS-2 distance outperforms freezing layers, while random layer mixing performs comparably to averaging. Layer mixing interpolates between neurons a top-k most similar layer pair. Hence, it can change the sign of some of the original incoming weights into the resulting mixed layer. Therefore, it is somewhat surprising that accuracy has remained relatively high, suggesting that similar layers have weights with a shared sign, not only a similar magnitude. 
%

\begin{figure}
\centering
 \includegraphics[width=1.\linewidth]{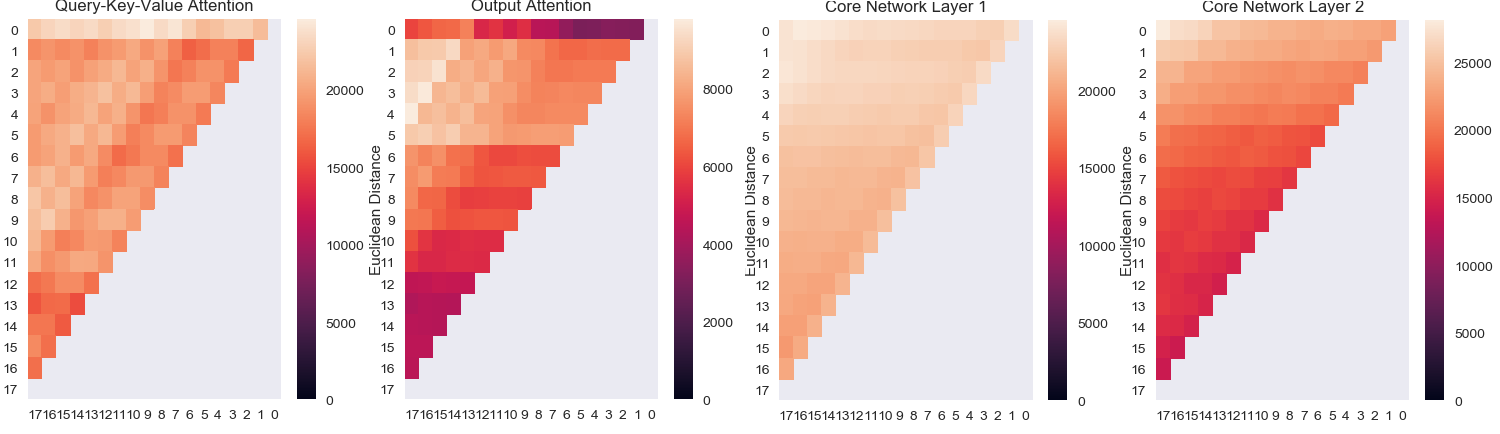} \caption{\small{Euclidean Distance Between Trans-XL Weights: (1) Query-Key-Value Attention, (2) Output Attention, (3, 4) FC Layers }}\label{fig:transxl_weight_sim}
\end{figure}

\subsection{Language Modelling}

\begin{figure}
\centering
\captionsetup{justification=centering,margin=1cm}
 \includegraphics[width=0.95\linewidth]{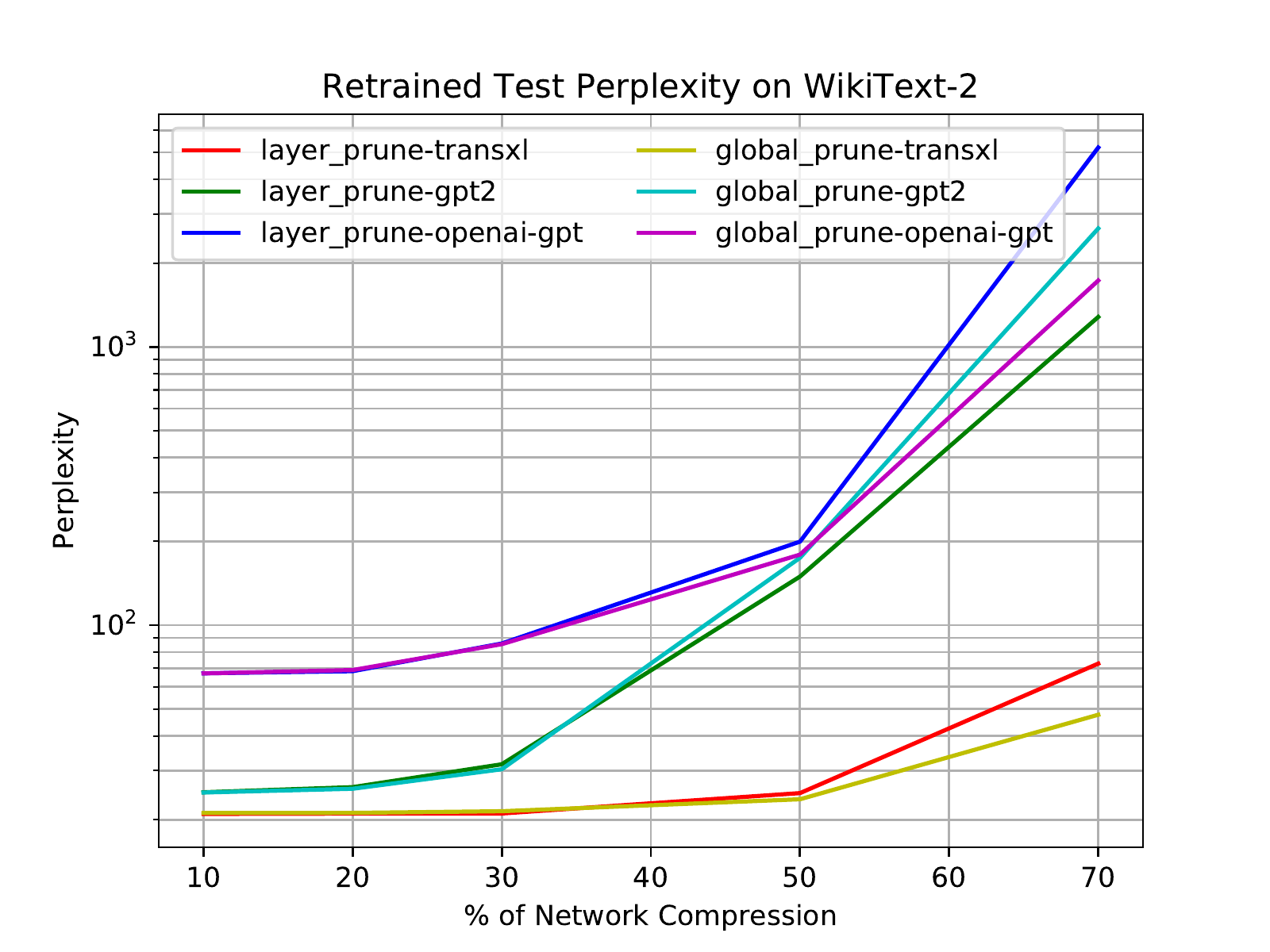}\caption{Wikitext-2 Language Modelling Pruning without Retraining}\label{fig:lm_untrained_results}
\end{figure}

We begin by showing the similarity between pretrained layers on Transformer-XL in Figure \ref{fig:transxl_weight_sim}, using sum of pairwise Euclidean distances. In general, we can see that closer layers have a smaller Euclidean distance. This more pronounced in the output attention (3) and fully connected layers (4) and slightly more sporadic among query-key-value attention weights (1).

\begin{figure*}
 \centering
 \subfigure[GPT]{\label{fig:gpt_retrain}\includegraphics[width=.33\textwidth]{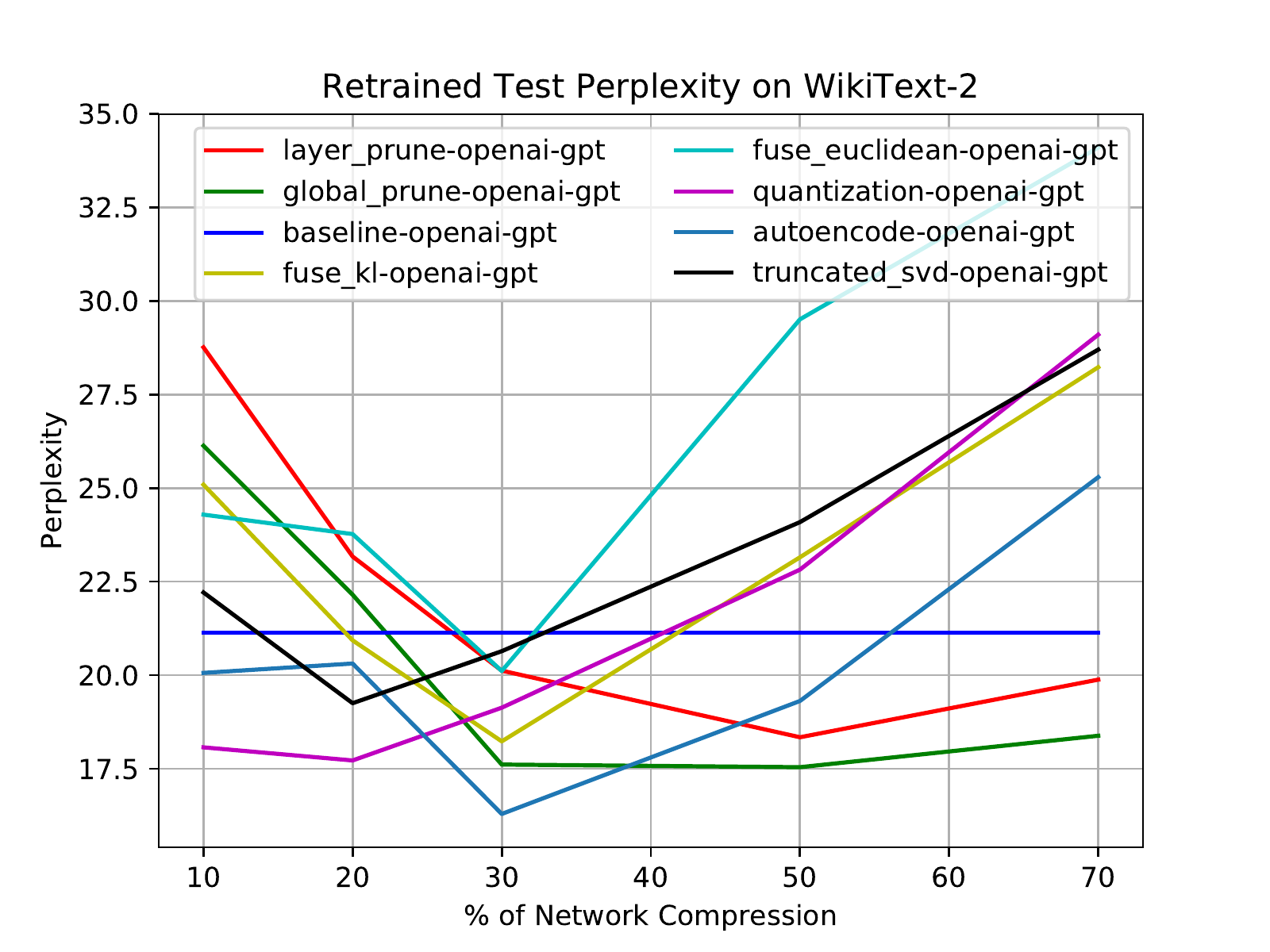}}
\subfigure[GPT-2]{\label{fig:gpt2_retrain}\includegraphics[width=.33\textwidth]{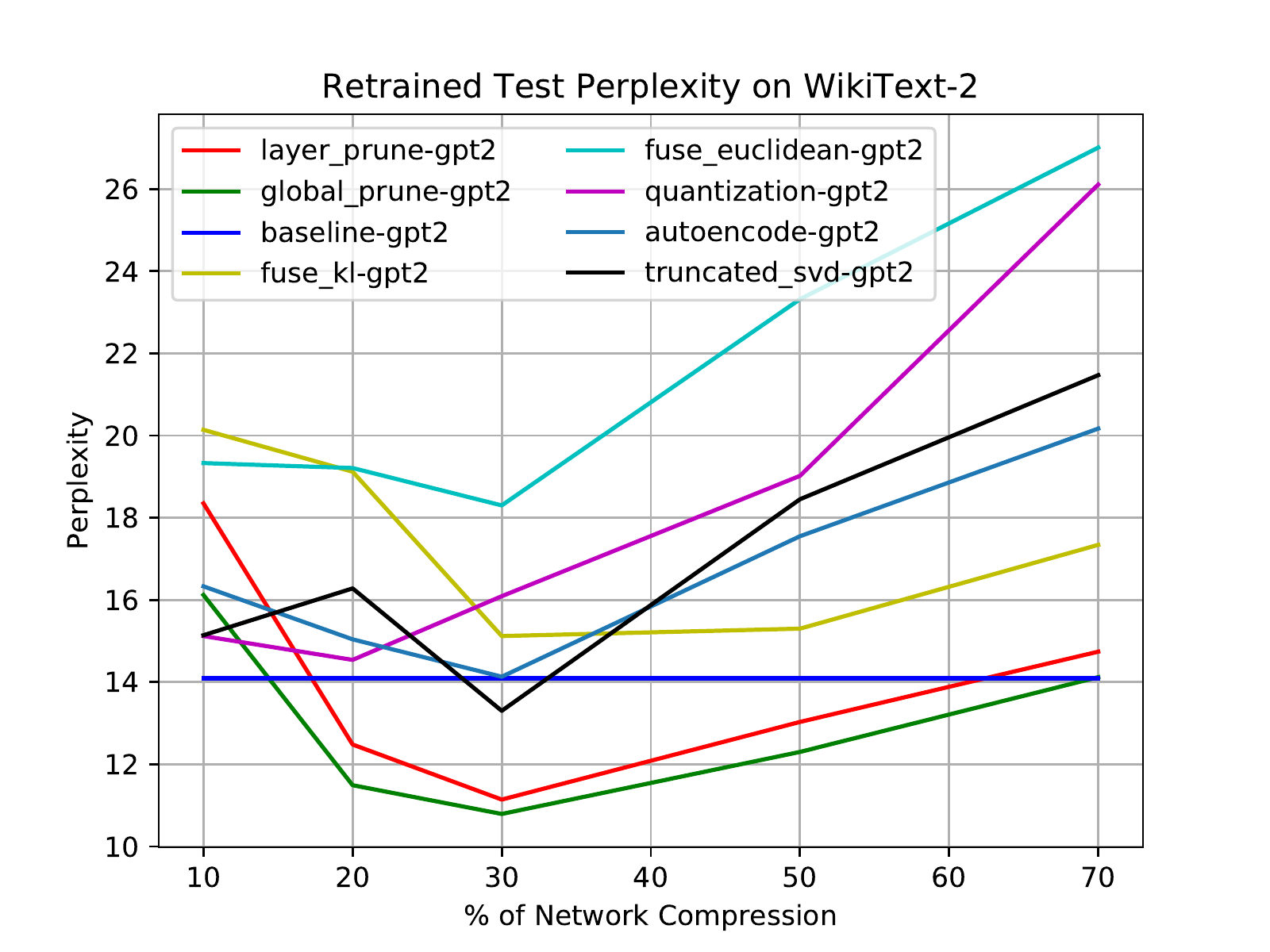}}
 \hfill
 \subfigure[Transformer-XL]{\label{fig:transxl_retrain}\includegraphics[width=.33\textwidth]{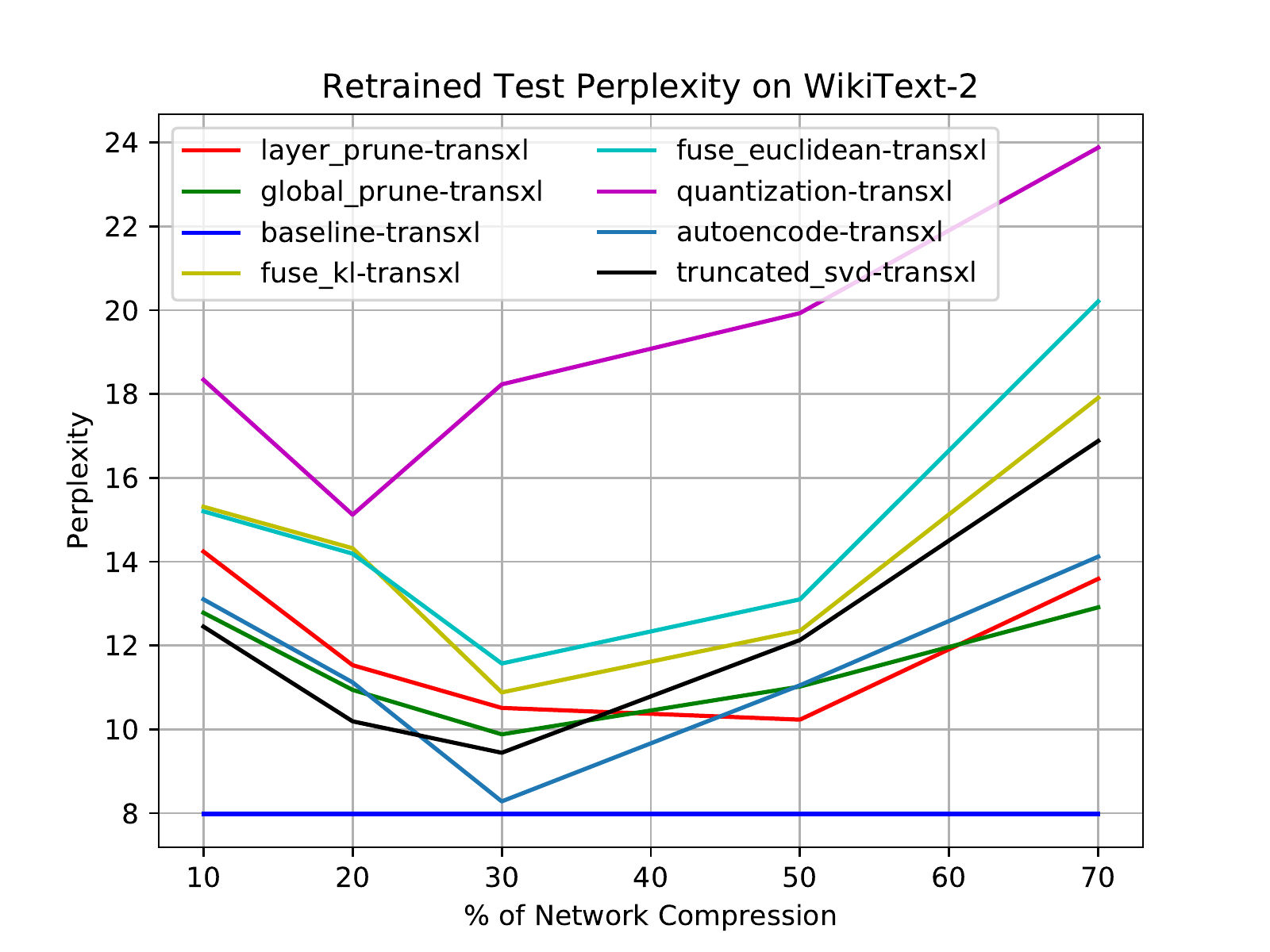}}
 \hfill
 \hfill
 \caption{Language Modelling Compression Results on WikiText-2 with Retraining}\label{fig:retrain_wiki}
\end{figure*}

In Figure \ref{fig:lm_untrained_results}, we find an exponential trend in perplexity (note the log-scale y-axis) increase with respect to the compression ratio for layer pruning and global pruning. Interestingly, Transformer-XL can maintain similar performance up to 50\% pruning from the pretrained model without any retraining. In contrast, we see that the original OpenAI-GPT is more sensitive and begins to show an exponential increase at 30\%. This insight is important for choosing the intervals between compress steps during iterative pruning, likewise for LF and tensor decomposition. Concretely, we would expect that the larger the increase in perplexity between compression steps, the more retraining epochs are needed. 
We also posit that this monotonically increasing trend in compression is related to the double descent phenomena~\citep{belkin2019reconciling}, whereby when more data is added or the model complexity is reduced, the network can fall back into the critical regime region~\citep{nakkiran2019deep} and even further into the underparameterized regime. This is reinforced by the fact that a large network such as Transformer-XL contains a smaller global weight norm of fully-connected layers in comparison to GPT and is able to maintain similar performance up to 50\% without retraining. 
Therefore, instead of choosing a constant $\%/N_c$ amount of compression at each compression step $N_c$, we allocate more compression earlier in retraining but more retraining steps later.

Figure \ref{fig:retrain_wiki} shows subfigures of retraining GPT (\ref{fig:gpt_retrain}), GPT-2 (\ref{fig:gpt2_retrain}) and Transformer-XL (\ref{fig:transxl_retrain}) with all aforementioned compression methods for GPT, GPT2 and Transformer-XL respectively. Firstly, we find retraining with a sufficient number of compression steps to be worthwhile for drastically reducing the network size while maintaining performance for both structured and unstructured approaches. Past 30\% of network reduction wee find a weakly linear increase, in contrast to the exponential increase with no retraining. We find that global pruning generally outperforms layer pruning as it doesn't restrict the percentage of weights pruned to be uniform across layers. This suggests that many layers are heavily pruned while others are preserved. This also coincides with findings from ~\citep{zhang2019all} that some layers are critical to maintain performance while removing the remaining layers has little effect on generalization. 
 %

 Table \ref{tab:fusion_criteria_results} shows the results of LF for compression ratio of 2 using layer averaging (Mean), layer freezing (Freeze) and mixing layers (-Mix) when ranking weight similarity using Euclidean distance (ED), KL and WS distance and CA. For all models CA produces the best results, slightly outperforming WS.


\begin{table}
\centering
\captionsetup{justification=centering, margin=0cm}
 \caption{WikiText-2 Perplexity after LF-Retraining (networks reduced to 50\% their original size)}\label{tab:fusion_criteria_results}
\resizebox{.7\linewidth}{!}{%
\begin{tabular}{lc|cccc}
\toprule[1.pt]

& & Mean & Freeze & Mix \\
\midrule
\multicolumn{2}{l|}{\textbf{TransXL-KL}} & 12.23 & 15.02 & 13.75 \\
\multicolumn{2}{l|}{\textbf{TransXL-ED}} & 13.08 & 17.13 & 14.88 \\
\multicolumn{2}{l|}{\textbf{TransXL-WS}} & 11.48 & 14.40 & 12.17 \\
\multicolumn{2}{l|}{\textbf{TransXL-CA}} & \textbf{11.13} & 13.97 & 14.73 \\

\midrule 
\multicolumn{2}{l|}{\textbf{GPT2-KL}} & 15.56 & 19.04 & 15.87 \\
\multicolumn{2}{l|}{\textbf{GPT2-ED}} & 16.03 & 21.14 & 16.73 \\
\multicolumn{2}{l|}{\textbf{GPT2-WS}} & \textbf{13.71} & 18.31 & 13.58 \\
\multicolumn{2}{l|}{\textbf{GPT2-CA}} & 14.03 & 21.28 & 14.59 \\

\midrule 
\multicolumn{2}{l|}{\textbf{GPT-KL}} & 23.57 & 28.01 & 24.82 \\
\multicolumn{2}{l|}{\textbf{GPT-ED}} & 25.07 & 29.68 & 24.73 \\
\multicolumn{2}{l|}{\textbf{GPT-WS}} & 19.10 & 23.17 & 18.90 \\
\multicolumn{2}{l|}{\textbf{GPT-CA}} & \textbf{18.48} & 22.01 & 20.39 \\

\bottomrule[1pt]
\end{tabular}%
}
\end{table}

\subsection{Additional Observations}
In language modelling, the effects of model reduction typically follow an exponential increase in perplexity for a compression ratio greater than 2 (corresponding to @50\%) when no retraining steps are used. Unlike CIFAR10 image classification, language modelling is a structured prediction task that has a relatively large output dimensionality which we posit has an important role in the amount of compression that can be achieved. ~\citet{yang2017breaking} have noted the \textit{softmax bottleck} whereby the restriction on the size of the decoder results in information loss when calibrating the conditional distribution, while ~\citep{belkin2019reconciling} have also noted the double descent phenomena is dependent on the number of classes. We conjecture that pruning and other such methods can exacerbate this bottlenecking and therefore the compression ratio will be generally lower compared to classification problems with relatively less classes, such as CIFAR-10. 

\section{Conclusion}

In this paper we proposed layer fusion, a new method for model compression. We find that merging the most similar layers during the retraining process of already deep pretrained neural network leads to competitive performance when compared against the original network, while maintaining a dense network. Layer fusion is also competitive with pruning, layer decomposition and knowledge distillation without the use of any additional parameters. We also find that mixing weight matrices during layer fusion performs comparably to layer averaging. Secondly, we compared how much compression can be achieved with and without retraining for both tasks and the importance of the number of epochs and compression steps. By using an exponential curriculum schedule to allocate the percentage of compression at each compression step, we find improvements over distributing the compression percentage uniformly during retraining. Lastly, a compression inflection point was observed in both tasks where the performance rapidly decreases, found for all compression methods and models.

\bibliography{aaai}
\bibliographystyle{aaai}

\end{document}


%
\title{Supplementary Material}

\maketitle

\section{Dataset and Model Description}
We focus on transformer-based models for natural language processing tasks given their recent successes in various applicable domains. For large models in NLP such as BERT~\cite{devlin2018bert}, OpenAI-GPT, GPT2~\cite{radford2018improving} and Transformer-XL~\cite{yang2019xlnet}, we freeze or combine layer weights of each multi-attention head component and intermediate dense layers, thus, dramatically reducing the respective number of layer and weights.

For language modelling (LM), we use the well-established WikiText-2 dataset~\cite{merity2016pointer}.
For GPT and GPT-2, the context length is 1024, training batch size of 4, learning rate is 6.25e-5 with a warm-up linear schedule and weight decay of 0.01.  
In the Transformer-XL model we use an adaptive softmax with cutoffs at (1000, 2500, 7000) for a vocabulary size of 28996 corresponding to vocabulary indices. The embedding size is 1024, 16 attention heads with dimensionality 64, the core 18 fully-connected layers are of dimensionality 4096, memory length is 1600 with target length 128. For all 3 models we train for 3 full epochs, including 20 retraining steps.
For image classification on CIFAR10, we focus on 4 CNN architectures of CNNs that are widely used for image classifiation: ResNet~\cite{he2016deep}, ResNet-ELU~\cite{shah2016deep}, Wide-ResNet~\cite{Zagoruyko2016WRN} and DenseNet~\cite{huang2017densely}. We are particularly interested in ResNet architectures and related ones that also use skip connections. This is motivated by ~\citet{veit2016} which found that deleting or permuting residual blocks can be carried out without much degradation in performance in a pretrained ResNet. 

\section{Compression Details}

\begin{figure*}
\centering
 \includegraphics[width=.8\linewidth]{images/transfo_xl_weight_sim_4b1.png}  \caption{Euclidean Distance Between Trans-XL Weights: (1) Query-Key-Value Attention, (2) Output Attention, (3, 4) FC Layers }\label{fig:transxl_weight_sim}
\end{figure*}


\subsection{Weight Pruning} We prune a percentage of the weights with the lowest magnitude. This is done in one of two ways: a percentage of weights pruned by layer (layer pruning), or a percentage of the network as whole (global pruning). When used in CNNs, this carried out on convolutional layers. 

\subsection{Layer Quantization}
In the subsequent figures that show k-means quantization results, percentages  correspond to the number of clusters as a percentage of the number of dimensions in that layer (e.g Quantizing GPT-2 layer $\theta_l$ at 50\% results in 512 clusters since $|\theta_{\ell}|=1024$). Each weight $\theta^{i}_{\ell}$ of the i-th weight of layer $\ell$ is assigned to cluster center $\vec{c}_{\ell} \in C$. 
For CNNs, all convolutional filters within a layer are flattened $\R^{c \times f} \to \R^{cf}$ for $c$ channels and $f$ filters. We then assign the cluster centers and reshape to the original dimensions $\R^{cf} \to \R^{c \times f}$ prior to pooling. The standard k-means objective is minimized as $\argmin_{C} \sum_{\ell=1}^{L}\sum_{i=1}^{k} |\vec{W}_{\ell}^{i} - \vec{c}_{\ell}^i|^{2}$~\cite{lloyd1982least}.

\subsection{Model Distillation}
Our approach to model distillation is different from that described by~\cite{hinton2015distilling}. For very deep pretrained models, training a smaller student network can exhibit slow convergence given that there is relatively little information given to teacher, only the output of the teacher network. Thus, we propose to instead reconstruct individual layers of the teacher network by iteratively reconstructing each layer through dimensionality reduction from the bottom of the architecture to the top. After learning to reconstruct layer $l$ we replace it with the lower dimensional layer before learning to reconstruct layer $l+1$. This conditioning in reconstruction ensures that model is behaving similarly to the teacher network, as opposed to learning to reconstruct each layer independently where past layers are not those from the student network but from the teacher network. We refer to this type of model distillation as \textit{student rollout} because the student network is iteratively rolled out from the first layer to the last. We consider denoising autoencoders and truncated SVD for student rollout.

Firstly, we use 1-hidden layer denoising autoencoders (DAE), as shown in \autoref{eq:ae_model_dist}, to reconstruct the self-attention blocks in Transformers (GPT, GPT2 \& Transformer-XL) and convolutional layers in ResNets (ResNet, ResNet-ELU \& Wide-ResNet) and DenseNets. Here $\omega$ and $\alpha$ are the weight matrix and bias layer reconstruction parameters respectively. For autoencoding self-attention layers, we re-normalize the attention layers using the $\mathtt{softmax}$ such that each row of the resulting $\mathbb{R}^{d_{\ell} \times d_{\ell}}$ matrix in the self-attention block sum to 1. We also re-normalize attention layers after pruning steps. The loss $\mathcal{L}(\mat{z}_{\ell}, \tilde{\mat{z}}_{\ell})$ is then defined as the Mean Squared Error (MSE) as $\mathcal{L}(\mat{z}_{l}, \tilde{\mat{z}}_{\ell}) = \argmin_{\omega_{\ell}} \mathbb{E}_{z}[\mathcal{L}_{2}(\mat{z}_{l}, \tilde{\mat{z}}_{\ell})]$.

\begin{gather}\label{eq:ae_model_dist}
\tilde{\vec{z}}_{\ell} = \text{ReLU}\big((\vec{z}_{\ell}+ \epsilon_{\ell})\mathbf{\omega_{(\ell, 1)}} + \alpha_{\ell}\big)\tp[-9]{\omega_{(\ell, 2)}}\quad \forall \ell \in L
\end{gather}

For transformers this is carried out for each layer within each self-attention block and for ResNet architectures, we use a denoising convolutional architecture and ignore the skip connections during reconstruction. 
The second model we use for layer reconstruction is a randomized truncated SVD~\cite{halko2011finding} where $\mat{U} \in \mathbb{R}^{n_{\ell}-1 \times n_{\ell}-1}$ is an orthogonal matrix of left singular vectors, $\vec{\Sigma} \in \mathbb{R}^{n_{\ell-1} \times n_{\ell}}$ diagonal matrix of singular values, and $\mat{V} \in \mathbb{R}^{\ell \times \ell}$ orthogonal matrix of right singular vectors.  We perform QR factorization on $\mat{W}_{\ell}$ such that $\mat{Q}^T_{\ell} \mat{W}_{\ell} = \mat{R}_{\ell}$ where $\mat{Q}_{\ell}$ are the orthonormal columns of $\mat{W}_{\ell}$. Randomized techniques are used to approximate the range of $\theta_{\ell}$ and reduce computation from $\mathcal{O}(\min(n_{\ell-1}n_{\ell}^2,n_{\ell-1}^2 n_{\ell}))$ to $\mathcal{O}(n_{\ell-1}n_{\ell} \log(k))$ where $k$ represents the approximate rank of $\theta_{\ell}$.

\subsection{Layer Fusion \& Compression ReTraining Details}

For our proposed LF approach, we consider the aforementioned distances, divergences and alignment measures to choose the top-k layer pairs for LF. For fusing, we consider layer averaging, layer mixing and layer freezing.  
%
Most aforementioned prior work that focuses on compression methods with retraining~\cite{han2015deep} and combinations thereof (e.g ~\cite{han2015deep} prune, quantize, followed by huffman coding) focus on solely on reducing the network size. In contrast, ~\cite{frankle2018lottery} focused on identifying \textit{lottery tickets} by identifying a subnetwork post-pruning and then retraining from the random initializations of the unpruned weights prior to training of that subnetwork. In contrast, we compress an already pretrained model for both CNNs and Transformers as our aim is to reduce the network while preserving network density and being close to the original performance, as opposed to identifying subnetworks .  
To maintain uniformity across each compression step, we prune, quantize, fuse and decompose a percentage of the weights for model reduction as opposed to thresholding the weight magnitudes. This ensures a consistent comparison across the compression methods (e.g thresholding weights in pruning does not have a direct equivalent to quantization or weight decomposition, unless we dynamically reduce the network size proportional to the number of weights pruned via thresholding). 
For retraining we consider two main schemes: (1) for each retraining step we carry out network compression (e.g via pruning) and then retrain remaining network and iteratively carried this out until the final epoch, or (2) in the case where network compression leads to non-zero weights (e.g weight sharing or LF), we freeze the network weights apart from those which have been identified for LF in which case we retrain before tying.

\paragraph{Fusing Layer of Unequal Size}
For a pair of vectorized tensors of column size $d$ and $d+k$, we remove $k$ weights that have the smallest magnitudes from the $2^{nd}$ tensor until both match. We also considered using PCA and SVD to fix the dimension for each layer pair. 
However, this is less of an issue in the case of convolutional layers since filters that are most similar tend to be in the same layer.
For transformer models, the same parts of each attention block are fused e.g the key weight tensor from one layer could not be fused with a value weight tensor from another, only another key weight tensor, and these are the same dimensions, hence the same length when vectorized.

\section{Approximating the Covariance Matrix}\label{sec:approx_cov}
For our experiments, some of the layers can be relatively large. For example, the large GPT-2 a weight matrix from a given hidden layer is $\mat{w} \in \mathbb{R}^{2048 \times 2048}$ and a total of 4,194,304 parameters. We split $\mat{W}$ into block matrices to perform covariance estimation. The  row $d_r$ and column $d+c$ dimensionality are restricted to $d_r \leq d_r \leq 128$. A submatrix $\mat{w}_{i, j}$ can be represented as,

\[
    \mat{w}_{i,j} = \begin{bmatrix} 
    \mat{w}_{1,1} & \dots & \mat{w}_{1, d_c} \\
    \vdots & \ddots  &  \\
    \mat{w}_{d_r, 1} & & \mat{w}_{d_r, d_c} 
    \end{bmatrix}
\]

and all submatrices of $\mat{W}$ are then formed as,

\[
\mat{W} = \begin{bmatrix} 
    \mat{w}_{1,1} & \mat{w}_{1,2} & \dots & \mat{w}_{1,m} \\
    \vdots & \ddots & &  \\
    \vdots & & \ddots &  \\
    \mat{w}_{n,1} & & & \mat{w}_{m,n} 
    \end{bmatrix}
\]

where $m = \text{round}(d_r / 128)$ and $n = \text{round}(d_c / 128)$. 

Given a pair of submatrices from two adjacent layers $\mat{w}^{n_{\ell}}_{i, j} \subset \mat{W}^{n_{\ell}}$ and $\mat{w}^{n_{\ell + 1}}_{i, j} \subset \mat{W}^{n_{\ell + 1}}$ from layers $n_{\ell}$ and $n_{\ell +1}$ respectively, we estimate the covariance similarity between. This is computed for all adjacent $m, n$ submatrix pairs, assuming that $d_r^{n_{\ell}} = d_r^{n_{\ell + 1}}$ and $d_c^{n_{\ell}} = d_c^{n_{\ell + 1}}$. The expected value of the full covariance similarity is then estimated by its mean as,

\begin{equation}
    \mathbb{E}[\Sigma_{\mat{W}}] = \frac{1}{mn}\sum_{i=1}^{m}\sum_{j = 1}^{n} d_{cov}\big(\Sigma_{\mat{w}^{n_{\ell}}_{i, j}},\Sigma_{\mat{w}^{n_{\ell + 1}}_{i, j}}\big)
\end{equation}

The covariance estimation techniques we discuss in \autoref{sec:sim_cov} are then applied to compute similarity between covariance matrices.

\section{Similarity Between Covariance Measures}\label{sec:sim_cov}
Here, we detail existing similarity measures for covariance matrices, which in the context of our work are the covariances of the a layers weights in a neural network. 


\paragraph{Riemmanian Metric} A basic similarity measure between covariances $\Sigma_{\mat{W}_1}, \Sigma_{\mat{W}_2}$ could be the vectorized $\mathtt{vec}(\Sigma_{\mat{W}_1}), \mathtt{vec}(\Sigma_{\mat{W}_2}) \in \mathbb{R}^{(d^2 + d)/2}$ and the $\ell_1$ or $\ell_2$ distance can be used for the Euclidean space. However, vectorization destructs the covariance structure of the weights from each neuron  (i.e  a row vector from $\Sigma_{\mat{W}}$). Ideally, we want to preserve which neuron a weight came from and maintain the manifold over the parameter space. Since the parameter covariance is positive semi-definite, it can be viewed as a Riemmanian manifold and we can consider a Riemannian Metric such as the Affine Invariant Riemannian Metric~\citep[AIRM;][]{barbaresco2013information,lawson2001geometric,pennec2006riemannian} expressed as,

\begin{equation}\label{eq:airm}
d_{\text{AIRM}}(\Sigma_{\mat{W}_1}, \Sigma_{\mat{W}_2}) := || \log(\Sigma_{\mat{W}_1}^{-1/2} \Sigma_{\mat{W}_2} \Sigma_{\mat{W}_1}^{-1/2})||_F
\end{equation}

where $F$ is the Frobenius norm of the resulting matrix. However, this can be slow to compute due to eigenvalue computation and the log of the parameter matrix, which for wide networks is more expensive.

The Log-Euclidean Metric~\citep{arsigny2005fast,arsigny2007geometric} shown in \autoref{eq:log_euc} reduces the computation by converting to symmetric matrices which can then 

\begin{equation}\label{eq:log_euc}
d_{\text{LERM}}(\Sigma_{\mat{W}_1}, \Sigma_{\mat{W}_2}) := || \log \big(\Sigma_{\mat{W}_1}\big) -  \log \big(\Sigma_{\mat{W}_2}\big)||_F
\end{equation}

However, the main drawback with the aforementioned AIRMs is the additional computation involved in preserving the curvature of the space of positive definite matrices.

The symmetrized KL-Divergence (KLDM) has also been used by combining the KL divergence in both directions as

\begin{equation}
\begin{split}
    d_{\text{SKL}}(\Sigma_{\mat{W}_1}, \Sigma_{\mat{W}_2}) = \frac{1}{2}\Big(\text{KL}[\Sigma_{\mat{W}_1} || \Sigma_{\mat{W}_2} ] + \\ \text{KL}[\Sigma_{\mat{W}_2} || \Sigma_{\mat{W}_1}]\Big)
\end{split}
\end{equation}

where $\Sigma_{\mat{W}}$ is vectorized. This overcomes the problem of KL not being symmetric and is invariant to inversion so that $d_{\text{KL}}(\Sigma_{\mat{W}_1}^{-1}, {\Sigma_{\mat{W}_1}}^{-1}) = d_{\text{KL}}(\Sigma_{\mat{W}_1}, {\Sigma_{\mat{W}_1}})$.

~\citet{cherian2011efficient} have introduced the Jenson

\subsection{Bures-Wasserstein Distance Between Covariances}
In our work we mainly focused on the Bures-Wasserstein Metric for directly computing the distance between weight matrices. However, an alternative approach is to minimize optimal transport between their covariances~\citep{bhatia2019bures}, which can be expressed as,

\begin{equation}
\begin{split}
    d_{WS-2}(\Sigma_{\mat{W}_1}, \Sigma_{\mat{W}_2}) = \Big[\text{tr}\Sigma_{\mat{W}_1} + \text{tr}\Sigma_{\mat{W}_2} - \\
    2 \text{tr}(\Sigma_{\mat{W}_1}^{1/2} \Sigma_{\mat{W}_2} \Sigma_{\mat{W}_1}^{1/2})^{1/2}\Big]^{1/2}
\end{split}
\end{equation}

We note that when we only use the variances from the diagonal of $\Sigma_{\mat{W}_1}$, the $d_{WS-2}$ becomes the Hellinger distance~\citep{hellinger1909neue} between the variances of each weight matrix.

\section{Canonical Correlation Analysis}
Covariance similarity compared intra-covariances of layers. An alternative, is to measure the cross-variances between weights in different layers. Canonical Correlation Analysis~\citep[CCA;][]{hotelling1935most} achieves this and can be used to find maximum correlation between linear cross-interactions of layers. 
~\citet{kornblith2019similarity} have used CCA between the activation outputs of layers from different networks. To reiterate, we are instead interested in weight similarity of the same network. 
%
The CCA $\rho$ coefficient can be obtained for any given layer pair as shown in \autoref{eq:rho_coef}. In cases where $|\mat{W}|$ is large, we employ the block matrix computation as described in \autoref{sec:approx_cov}, similarly performing an expectation over the constituent $\rho_{i, j}$ obtain for each adjacent submatrix of the $i$-th and j-$th$ layer. We define $\mat{U} = \vec{a'}\mat{W}_i$ and $\mat{V} = \vec{b}' \mat{W}_2$ where $\mat{W}_i \in \mathbb{R}^{n_i \times d_i}$, $\mat{W}_j \in \mathbb{R}^{n_j \times d_j}$, while $U$ and $V$ have unit variance. 

\begin{equation}\label{eq:rho_coef}
\rho_{i, j} = \frac{\text{Cov}(\mat{U}, \mat{V})}{\sqrt{\text{Var}(\mat{U})}\sqrt{\text{Var}(\mat{V})}} =
\frac{\vec{a}' \vec{\Sigma_{\mat{W}_i \mat{W}_j}} \vec{b}}{\sqrt{\vec{a}'\vec{\Sigma_{\mat{W}_i}}\vec{a}}{\sqrt{\vec{b}' \vec{\Sigma_{\mat{W}_j}} \vec{b}}}}
\end{equation}

This is carried out to compute $\rho_{i,j} \ \forall i, i+j \leq L$ and rank order the top-k most similar layers chosen for LF. We note that for the retraining step where CCA is used, dropout is not, because the zero entries inflate the $\rho$ coefficient and layer pairs that have the highest sum total of dropped neurons could potentially be chosen as the most similar even when their cross-variances are not.

\subsection{Nonlinear Extensions of CCA}

\paragraph{Kernel CCA}
CCA has also been extended with kernels for accounting for nonlinear cross-interactions~\citep[KCCA;][]{lai2000kernel}. We do not carry out experiments with KCCA but include a description for completeness here as it can potentially be used for LF. First, consider two Reproducing Kernel Hilbert Spaces (RKHS) $\mathcal{H}_1, \mathcal{H}_2$ with kernels $\kappa_1, \kappa_2$. The aim is to find functions $\phi_1 \in \mathcal{H}_1$ and $\phi_2 \in \mathcal{H}_2$ that maximize the correlation between weight matrices $\phi_1(\mat{W}_1)$ and $\phi_2(\mat{W}_2)$. This can be expressed as

\begin{gather}
(\phi_1, \phi_2) = \argmax_{\phi_1 \in \mathcal{H}_1, \phi_2 \in \mathcal{H}_2} \text{corr}(\phi_1(\mat{W}_1),f2(\mat{W}_2)) = \\ \nonumber
\argmax_{\phi_1 \in \mathcal{H}_1, \phi_2 \in \mathcal{H}_2} \frac{\text{cov}(\phi_1(\mat{W}_1),\phi_2(\mat{W}_2))}{\sqrt{\text{var}(\phi_1(\mat{W}_1)) \text{var}(\phi_2(\mat{W}_2))}}
\end{gather}

The ``kernel trick'' can be used since $\phi_1, \phi_2$ are both in RKHS and thus the solution can be found to solve the following optimization,


\begin{gather}
(\alpha_1,\alpha_2) = \argmax_{\alpha_1,\alpha_2} \frac{\alpha'_1 K^2_1 K_2 \alpha_2}{\sqrt{(\alpha'_1 K^2_1 \alpha_2) (\alpha'_1 K^2_2 \alpha_2)}} = \\ \nonumber
\argmax_{\alpha'_1 K^2_1 \alpha_1 = \alpha'_2 K^2_2 \alpha^2 = 1}\alpha'_1 K_1 K_2 \alpha_2    
\end{gather}

where $K_1 \in \mathbb{R}^{m \times m}$ is the centered Gram matrix $K_1 = K - K \mathbf{1} - \mathbf{1}K + \mathbf{1} K_1, K_{ij}= \kappa_1(x_i,x_j)$ and $\mathbf{1} \in \mathbb{R}^{m \times m}$ is an all-1s matrix, and similarly for $K_2$. Subsequent vectors $(\alpha^j_1,\alpha^j_2)$ are solutions of (7) with the constraints that $(f^j_1(X_1),f^j_2(X_2))$ are uncorrelated with the previous ones. Proper regularization may be critical to the performance of KCCA, since the spaces $\mathcal{H}_1,\mathcal{H}_2$ could have high complexity.  Since $\alpha'_1 \phi_1(\cdot)$ plays the role of $w_1$ in KCCA, the generalization of $w'1w_1$ would be $\alpha'_1 K_1 \alpha$. Therefore the correct generalization of (5) is to use $K21+r1K1$ in place of $K_{21}$ in the constraints of (7), for regularization parameter $r_1 > 0$ (resp. for $K^2_2$).The optimization is in principle simple:  The objective is maximized by the top eigenvectors of the matrix

\begin{equation}
(K_1 + r \mathbf{1}I) - \mathbf{1} K_2(K_2+ r_2 I) - \mathbf{1} K_1
\end{equation}

The  regularization  coefficients $r_1$ and $r_2$,  as  well as any parameters of the kernel in KCCA, can be tuned using held-out data. Often a further regularization isdone by first projecting the data onto an intermediate-dimensionality space, between the target and original dimensionality~\citep{ek2008ambiguity,arora2013multi}. In practice solving KCCA may not be straightforward, as the kernel matrices become very large for real-world data sets of interest, and iterative SVD algorithms forthe initial dimensionality reduction can be used~\citep{arora2012kernel}.

\paragraph{Deep CCA} Similarly, a deep neural network with   CCA~\citep[DCCA;][]{andrew2013deep} has been used for even richer representation and capture highly nonlinear cross-variances among input pairs. However, DCCA is slower to compute as the

The goal is to jointly learn parameters for both views $W^v_l$ and $b^v_l$ such that $\text{corr}(\phi_1(X_1),\phi_2(X_2))$ is as high as possible. If $\theta_1$ is the vector of all parameters $W^1_l$ and $b^1_l$ of the first view for $l= 1,\ldots,d$, and similarly for $\theta_2$, then 

\begin{equation}
(\theta*_1, \theta*_2) =\argmax_{(\theta_1,\theta_2)} \text{corr}(\phi_1(X_1;\theta_1),\phi_2(X_2;\theta_2))    
\end{equation}

To find $(\theta*_1,\theta*_2)$, we follow the gradient of the correlation objective as estimated on the training data. Let $\mathcal{H}_1 \in \mathbb{R}^{o \times m},\mathcal{H}_2 \in \mathbb{R}^{o \times m}$ be matrices whose columns are the top-level representations produced by the deep modelson the two views, for a training set of size $m$. Let $\hat{H}_1 = H_1 - \frac{1}{m} H_1 \mathbf{1}$ be the centered data matrix (resp. $\bar{H}_2$), and define $\hat{\Sigma}_{12} = \frac{1}{m-1} \bar{H}_1 \bar{H}'_2$ and $\hat{\Sigma}_{11} = \frac{1}{m-1} \bar{H}_1 \bar{H}'_1 + r_1 I$ for regularization constant $r_1$ (resp. $\Sigma^2_2$). Assume that $r_1 > 0$ so that $\Sigma_{11}$ is positive definite. As discussed in section 2 for CCA, the total correlation of the top $k$ components of $H_1$ and $H_2$ is the sum of the top $k$ singular values of the matrix $T=\hat{\Sigma}_{11}^{-1/2} \hat{\Sigma}_{12} \hat{\Sigma}_{22}^{-1/2}$. If we take $k=o$, then this is exactly the matrix trace norm of $T$, or 

\begin{equation}
\text{corr}(H_1,H_2) = ||T||_{\text{tr}} = \text{tr}(T'T)^{1/2}    
\end{equation}

The parameters $W^v_l$ and $b^v_l$ of DCCA are trained to1Here we abuse notation slightly, writing $\text{corr}(H_1, H_2)$ as the empirical correlation of the data represented by the matrices $H_1$ and $H_2$. optimize this quantity using gradient-based optimiza-tion.  To compute the gradient of $\text{corr}(H_1,H_2)$ with respect to all parameters $W^v_l$ and $b^v_l$, we can compute its gradient with respect to $H_1$ and $H_2$ and then use backpropagation. If the singular value decomposition of $T$ is $T=U D V'$, then

\begin{equation}
\frac{\partial \text{corr}(H_1,H_2)}{\partial H_1}= \frac{1}{m-1}\Big(2\nabla_{11}\bar{H}_1 + \nabla_{12} \bar{H}_2 \Big)  
\end{equation}

where

\begin{equation}
\nabla_{12} = \Sigma^{-1/2}_{11} U V' \hat{\Sigma}^{-1/2}_{22}     
\end{equation}

and 

\begin{equation}
\nabla_{11} = -\frac{1}{2}\hat{\Sigma}^{-1/2}_{11} U D U' \hat{\Sigma}^{-1/2}_{11}    
\end{equation}

and $\partial \text{corr}(H_1,H_2)/\partial H_2$
has a symmetric expression.The derivation of the gradient is not entirely straight-forward (involving,  for example, the gradient of the trace of the matrix square-root, which we could not find in standard references such as~\citet{petersen2008matrix} and is given in the appendix.  We also regularize (10) by adding to it a quadratic penalty with weight $\lambda_b > 0$ for all parameters.

\paragraph{Singular Value CCA (SVCCA)}

CCA is sensitive to perturbation when the condition number of X or Y is large~\citep{golub1995canonical}. To improve robustness, singular vector CCA (SVCCA) performs CCA on truncated singular value decompositions of $X$ and $Y$~\citep{raghu2017svcca,mroueh2015asymmetrically,kuss2003geometry}, SVCCA keeps enough principal components of the input matrices to explain a fixed proportion of the variance, and drops remaining components. Thus, it is invariant to invertible linear transformation only if the retained subspace doesnot change.

\paragraph{Projection-weighted CCA (PWCCA)} 
~\citet{morcos2018insights} propose a different strategy to reduce the sensitivity of CCA to perturbation, which they term ``projection-weighted canoni-cal correlation'' (PWCCA):

\begin{equation}
\rho_{\text{PW}} = \frac{\sum_{i=1}^{c} \alpha_i \rho_i}{\sum_{i=1} \alpha_i}, \quad \alpha_i = \sum_j \langle h_i, x_j \rangle    
\end{equation}

where, $x_j$ is the j-th column of $X$, and $h_i = X w_iX$ is the vector of canonical variables formed by projecting $X$ to the i-th canonical coordinate frame.  As we show in Appendix C.3, PWCCA is closely related to linear regression, since:

\begin{equation}
R^2_{LR}= \frac{\sum^c_{i=1}\alpha'_i \rho^2_i}{\sum_{i=1}\alpha'_i}, \quad \alpha'_i = \sum_j \langle h_i, x_j \rangle^2    
\end{equation}

\subsection{Centred Kernel Alignment}

~\citet{kornblith2019similarity} have proposed centred kernel alignmnet (CKA) which 

Linear CKA is closely related to CCA and linear regression.If $X$ and $Y$ are centered, then $\mathcal{Q}_X$ and $\mathcal{Q}_Y$ are also centered, so:

\begin{equation}
R^2_{\text{CCA}} = \text{CKA}(\mathcal{Q}_X \mathcal{Q}^T_X, \mathcal{Q}_Y \mathcal{Q}^T_Y) \sqrt{\frac{p_2}{p_1}}    
\end{equation}

When performing the linear regression fit of $X$ with design matrix $Y,R^2_{\text{LR}} = ||\mathcal{Q}^T_Y X||^2_F / ||X||^2_F$, so:

\begin{equation}
R^2_{\text{LR}} = \text{CKA}(X X^T, \mathcal{Q}_Y \mathcal{Q}^T_Y) \frac{\sqrt{p_1}||X^T X||_F}{||X||^2_F}    
\end{equation}

When might we prefer linear CKA over CCA? One wayto show the difference is to rewrite $X$ and $Y$ in terms of their singular value decompositions $X=U_X \Sigma_X V^T_X, Y = U_Y \Sigma_Y V^T_Y$.  Let the $i-$th eigenvector of $X X^T$ (left-singular vector of X) be indexed as $u^i_X$. Then $R^2_{\text{CCA}}$ is:

\begin{equation}
R^2_{\text{CCA}} = ||U^T_Y U_X||^2_F /p_1 = \sum_{i=1}^{p_1}\sum_{j=1}^{p_2}\langle u^i_X, u^j_Y \rangle^2 /p_1   
\end{equation}

Let the $i$-th eigenvalue of $X X^T$ (squared singular value of $X$) be indexed as $\lambda^i_X$. Linear CKA can be written as:

\begin{gather}
\text{CKA}(X X^T, Y Y^T) =|\frac{|YTX||2F}{||XTX||F||YTY||F} = \\ \nonumber \frac{\sum^{p_1}_{i=1}\sum^{p_2}_{j=1}\lambda^i_X \lambda^j_Y \langle u^i_X, u^j_Y\rangle^2}{\sqrt{\sum_{i=1}^{p_1}(\lambda^i_X)^2}\sqrt{\sum_{j=1}^{p_2}(\lambda^j_Y)^2}}  
\end{gather}

Linear CKA thus resembles CCA weighted by the eigenvalues of the corresponding eigenvectors, i.e.the amountof variance in $X$ or $Y$ that each explains. SVCCA (Raghuet al., 2017) and projection-weighted CCA (Morcos et al.,2018) were also motivated by the idea that eigenvectors that correspond to small eigenvalues are less important, bu

\subsection{Neuron Alignment}
Neuron Alignment Procedures.Other work has studied alignment between individual neurons, rather than alignment between subspaces. Li et al. (2015) examined correlation between the neurons in different neural networks, and attempt to find a bipartite match or semi-match that maximizes the sum of the correlations between the neurons, and then to measure the average correlations. Wang et al. (2018) proposed to search for subsets of neurons $\tilde{X} \subset X$ and $\tilde{Y} \subset Y$ such that, to within some tolerance, every neuronin $\tilde{X}$ can be represented by a linear combination of neurons from $\tilde{Y}$ and vice versa. They found that the maximum matching subsets are very small for intermediate layers.

\section{Cherians Description}
\subsection{Riemannian Metric}
The simplest but naive approach is to view $d\times d$ covariance matrices as vectors in $\mathbb{R}^{d(d + 1)/2}$, whereby the standard (dis)similarity measures of Euclidean space can be used (e.g.,$\ell_p$-distance functions, etc.). But these vectorial measures ignore the manifold structure of covariance matrices and can therefore be inferior choices. 

A more suitable choice is to consider the manifold structure of positive-definite matrices and use the corresponding geodesic distance:  the Affine Invariant Riemannian Metric (AIRM) [5,28] defined for $X, Y$ in $S^{d}_{++}$, the set of $d \times d$ positive-definite matrices, by:

\begin{equation}
D_R(X, Y) := || \log(X^{-1/2} Y X^{-1/2})||_F
\end{equation}

where $\log(\cdot)$ is the matrix logarithm. This metric enjoys several useful theoretical properties [5], and is perhaps the most widely used similarity measure for covariance matrices. But it can be unattractive as it requires eigenvalue computations or sometimes even matrix logarithms, which for larger matrices causes significant slow downs. Amongst the many measures that have been proposed to replace AIRM, a closely related one is the Log-Euclidean Riemannian Metric (LERM) [2] 

\begin{equation}
D_{LE}(X, Y) := || \log(X) - \log(Y)||_F   
\end{equation}

which uses the logarithmic map of covariance matrices to turn them into symmetric matrices which can then be handled as objects in ordinary Euclidean space. Applications exist in visual tracking [18] and stereo matching [14]. However, using this metric requires pre-processing the dataset by computing matrix logarithms, which can dramatically increase the computational costs. Yet another alternative is the symmetrized KL-Divergence Measure (KLDM) for positive-definite matrices [20], though for our application its accuracy on NN is poor. Other similarity measures on covariance matrices may be found in [11]. In contrast to AIRM and LERM the similarity measure that we propose is much faster to compute, as it depends only on the determinant of the input matrices, and thus no eigenvalue computations are required. Moreover, as we will later see, it turns out to be empirically also very effective.These gains come at a price: our measure is not a metric.But this limitation is not that severe because we can still exploit convexity to build a fast NN retrieval procedure based on our similarity measure.We note that NN retrieval for covariance matrices itself is still an emerging area, so literature on it is scarce. In [30], an attempt is made to adapt NN techniques from vector spaces to non-Euclidean spaces, while [9] proposes an extension of the spectral hashing techniques to covariance matrices.However, both these techniques are based on a Euclidean embedding of the Riemannian manifold through the tangent spaces, and then using LERM as an approximation to the true similarity.

\subsection{Jensen-Bregman LogDet divergence}
We first recall some basic definitions and then presentour new similarity measure: theJensen-Bregman LogDetDivergence (JBLD). We remark that although this measureseems natural and simple, to our knowledge it hasnot beenformally discussed in detail before. We alert the reader thatJBLD should not be confused with its asymmetric cousin:the so-called LogDet divergence [15].

At the core of our discussion lies the Bregman Divergence $d_\phi: S\times \text{relint}(S) \to [0,\infty)$, which is defined as 

\begin{equation}
d_{\phi}(x, y) :=\phi(x) - \phi(y) - \langle x - y,\nabla \phi(y) \rangle    
\end{equation}

where $\phi: S \subseteq \mathbb{R}^d \to \mathbb{R}$ is a strictly convex function of Legendre-type on int($\mathtt{dom} S$) [4,6]. From (3) the following properties of $d_{\phi}(x, y)$ are apparent: strict convexity in $x$;asymmetry; non-negativity; and definiteness (i.e.,$d_{\phi} = 0,\ \text{iff} \ x=y$). Bregman divergences enjoy a host of useful properties [4,8], but often their lack of symmetry and sometimes their lack of triangle-inequality can prove to be hindrances. Consequently, there has been substantial interest in considering symmetrized versions such as Jensen-Bregman divergences [3,22,23], where assuming $s= (x+y)/2$,

\begin{equation}\label{eq:4}
 J_{\phi}(x, y) :=\frac{1}{2}\big(d_{\phi}(x, s) +d_{\phi}(s, y)\big)    
\end{equation}

or even variants that satisfy the triangle inequality [3,10]. Both (3) and (4) can be naturally extended to matrix divergences (over Hermitian matrices) by composing the convex function $\phi$ with the eigenvalue map $\lambda$, and replacing theinner-product in (3) by the trace. We focus on a particular matrix divergence, namely the Jensen-Bregman LogDet Divergence, which is defined for $X,Y$ in $S^d_{++}$ by 

\begin{equation}
J_{\ell d}(X, Y) := \log\Big|\frac{X + Y}{2}\Big| - \frac{1}{2} \log |XY|
\end{equation}

where $|\cdot|$ denotes the determinant; this divergence is obtained from the matrix version of \autoref{eq:4} by using $\phi(X) = - \log|X|$ as the seed function. It is easy to see that $J_{\ell d}$ is symmetric, nonnegative,  and definite. Moreover, it is invariant under congruence transformations, $(J_{\ell d}(A X A^T, AY A^T) =J_{\ell d}(X, Y)$ for invertible $A$, and under inversion $(J_{\ell d}(X, Y) =J_{\ell d}(X^{-1}, Y^{-1}))$. Less trite is the connection to the Riemannian metric which we summarize below in Theorem 1. This connection also lends additional support to using $J_{\ell d}$ as a proxy for the AIRM $D^2_R$ (or $\sqrt{J_{\ell d}}$ as a proxy for $D_R$).

\begin{theorem}[Bounds]
 Let $X,Y \in S^d_{++}$. Then,

\begin{equation}
J_{\ell d}(X, Y) \leq D^2_R(X, Y)  
\end{equation}

Additionally, if $0 \prec mI \cdot X, Y \prec M I$, then 

\begin{equation}
D^2_R(X, Y) \leq 2 l\log(M/m)(J_{\ell d}(X, Y) +\gamma)    
\end{equation}

where $\gamma=d \log 2$. 
\end{theorem}

\begin{proof}
Let $v_i=\lambda_i(X Y^{-1})$.  Since $X,Y \in S_d^{++}$, the eigen values $v_i$ are also positive, whereby we can write each $v_i= e \vec{u}_i$ for some $\vec{u}_i \in \mathbb{R}$. Using this notation, the AIRM may be rewritten as $D_R(X, Y) =||u||_2$, and the JBLD as

\begin{equation}
J_{\ell d}(X, Y) = \sum^d_{i=1}(\log(1 +e^{\vec{u}_i}) - \vec{u}_i/2 - \log 2)    
\end{equation}

where the equation follows by observing that

\begin{equation}
J_{\ell d}(X, Y) = \log |I + X Y^{-1}| -  \frac{1}{2}\log|XY^{-1}| - \log 2d.    
\end{equation}

To prove inequality (6) consider the function $f(u) = u^2 - \log(1 +e^u) + u/2 + \log 2$. This function is convex since its second derivative

\begin{equation}
f''(u) = 2 - \frac{e^{u}}{(1 +e^u)^2}    
\end{equation}

, is clearly nonnegative. Moreover, $f$ attains its minimum at $u^*= 0$, as is immediately seen by solving the optimality condition 

\begin{equation}
f'(u) = 2u - \frac{e^u}{(1 +e^u)} + \frac{1}{2} = 0    
\end{equation}

Thus, $f(u) \geq f(u *) = 0 \quad \forall u \in R$, which in turn implies that 

\begin{equation}
\sum^d_{i=1} f(u_i) = D^2_R(X, Y) - J_{\ell d}(X, Y) \geq 0    
\end{equation}

To prove the next inequality (7), first observe that 

\begin{equation}
    \sum^d_{i=1}(\log(1 + e^{\vec{u}_i})- \vec{u}_i/2 - \log 2) \geq \sum^{d}_{i=1}(|\vec{u}_i|/2 - \log 2),
\end{equation}

which implies the bound

\begin{equation}
J_{\ell d}(X, Y) +d\log 2 \geq \frac{1}{2} ||u||_1    
\end{equation}

Since $u^T u \leq ||u||_{\infty}||u||_1$ (Holds inequality), using (10) we immediately obtain the bound

\begin{equation}
D^2_R(X, Y) = ||u||^2_2 \leq 2||u||_{\infty}(J_{\ell d}+\gamma)   
\end{equation}

where $\gamma=d \log 2$. But $mI \cdot X, Y \cdot MI$ implies that $||u||_{\infty} \leq \log(M/m)$, which concludes the proof.
\end{proof}

\bibliography{aaai}
\bibliographystyle{aaai}